\documentclass[preprint,3p,times]{elsarticle}
\RequirePackage{multirow,booktabs,subfigure,color,array,hhline,makecell}
\usepackage{amssymb}
\usepackage{amsmath}
\usepackage{graphicx}
\usepackage{amsthm}
\usepackage{mathrsfs}
\usepackage{indentfirst}
\usepackage[colorlinks,citecolor=blue,urlcolor=blue]{hyperref}
\allowdisplaybreaks
\usepackage[table]{xcolor}
\usepackage{tikz-network}
\usepackage{pgf}
\usepackage{tikz}
\usepackage{subfigure}
\usepackage[T1]{fontenc}\usepackage[utf8]{inputenc}
\usetikzlibrary{arrows, decorations.pathmorphing, backgrounds, positioning, fit, petri, automata}
\usetikzlibrary{shadows,arrows,positioning}
\RequirePackage{algorithm,algpseudocode}
\allowdisplaybreaks
\usepackage{changes}

\usepackage{setspace}
\newtheorem{thm}{Theorem}
\newtheorem{defin}{Definition}
\newtheorem{lem}{Lemma}
\newtheorem{assum}{Assumption}

\newtheorem{cor}{Corollary}

\allowdisplaybreaks[4]
 
\AtBeginDocument{%
	\providecommand\BibTeX{{%
			\normalfont B\kern-0.5em{\scshape i\kern-0.25em b}\kern-0.8em\TeX}}}
\journal{~}
\begin{document}
\begin{frontmatter}
\title{Individual-heterogeneous sub-Gaussian Mixture Models}
\author[label1]{Huan Qing\corref{cor1}}
\ead{qinghuan@u.nus.edu~\&~qinghuan@cqut.edu.cn}
\cortext[cor1]{Corresponding author.}
\address[label1]{School of Economics and Finance, Chongqing University of Technology, Chongqing, 400054, China}

\begin{abstract}
The classical Gaussian mixture model assumes homogeneity within clusters, an assumption that often fails in real-world data where observations naturally exhibit varying scales or intensities. To address this, we introduce the individual-heterogeneous sub-Gaussian mixture model, a flexible framework that assigns each observation its own heterogeneity parameter, thereby explicitly capturing the heterogeneity inherent in practical applications. Built upon this model, we propose an efficient spectral method that provably achieves exact recovery of the true cluster labels under mild separation conditions, even in high-dimensional settings where the number of features far exceeds the number of samples. Numerical experiments on both synthetic and real data demonstrate that our method consistently outperforms existing clustering algorithms, including those designed for classical Gaussian mixture models.
\end{abstract}

\end{frontmatter}
\section{Introduction}\label{sec:intro}
Clustering is a fundamental task in modern statistics and machine learning. Its applications span diverse domains, including the identification of cancer subtypes from gene expression profiles, customer segmentation in marketing databases, community detection in social networks, disease diagnosis in medical imaging, user preference modeling in recommender systems, and pixel classification in satellite imagery. In nearly all fields that generate high‑dimensional measurements, clustering algorithms play an essential role in uncovering latent group structures \citep{jain1999data,jain2010data,fortunato2010community,fortunato2016community}. As datasets grow in size and complexity, the demand for clustering methods that are both statistically sound and computationally feasible becomes increasingly urgent.

Among probabilistic models for clustering, the Gaussian mixture model (GMM) has a long history, dating back to the analysis of crab morphometric data by \citep{pearson1894iii}. In the GMM, each observation is generated by first selecting a latent class and then drawing a vector from a class‑specific Gaussian distribution. Any continuous distribution can be approximated arbitrarily well by a finite mixture of Gaussians. This property underlies the model's extensive use in density estimation, clustering, and classification. A substantial body of work has addressed the estimation of Gaussian mixtures and the recovery of cluster labels. The expectation‑maximization (EM) algorithm \citep{dempster1977maximum} remains a standard iterative method for maximum likelihood estimation, and its convergence properties have been thoroughly analyzed \citep{wu1983convergence,xu2016global,balakrishnan2017statistical}. The K‑means algorithm \citep{lloyd1982least,mcqueen1967some} can be understood as the hard‑clustering limit of EM for spherical Gaussians with equal mixing proportions. Spectral clustering methods \citep{ng2001spectral,von2007tutorial,chen2021spectral} are also widely used: they reduce dimensionality by extracting the leading eigenvectors of a similarity matrix, after which a clustering routine is applied.

In recent years, the theoretical understanding of clustering under the GMM has been advanced dramatically. \citep{lu2016statistical} established statistical and computational guarantees for Lloyd's algorithm, demonstrating that with a suitable initialization, it achieves an exponential rate of misclassification error. \citep{loffler2021optimality} proved that spectral clustering itself is minimax optimal in the Gaussian mixture model, provided the number of clusters is fixed. \citep{abbe2022lp} extended the analysis to sub‑Gaussian mixtures, obtaining exponential misclassification rates for spectral clustering and, for the two‑component symmetric case, an explicit sharp threshold for exact recovery. \citep{zhang2024leave} further developed a leave‑one‑out perturbation theory, achieving the optimal exponential rate for spectral clustering under general sub‑Gaussian mixture models. \citep{chen2021cutoff} derived a sharp cutoff for exact recovery using a semidefinite programming relaxation of k‑means, while \citep{ndaoud2022sharp} obtained the optimal threshold for the two‑component case with a hollowed Gram matrix approach. \citep{li2025exact} further established sharp exact recovery thresholds for Gaussian mixtures with entrywise heterogeneous noise and arbitrary community sizes. \citep{giraud2019partial} provided partial recovery bounds for the relaxed k‑means, and \citep{fei2018hidden} studied hidden integrality of SDP relaxations for sub‑Gaussian mixtures. More recently, \citep{chen2024achieving} extended the optimal clustering theory to anisotropic Gaussian mixtures where covariances differ across clusters, while \citep{srivastava2023robust,jana2025adversarially} developed robust clustering algorithms that achieve optimal mislabeling rates in the presence of adversarial outliers. Collectively, these works have greatly deepened our understanding of when and why clustering algorithms succeed under the GMM.

Despite these advances, nearly all existing work on Gaussian or sub‑Gaussian mixture models relies on a crucial homogeneity assumption: observations within the same cluster are presumed to share the same scale or intensity. That is, the noise level and signal magnitude are taken to be identical for all data points belonging to the same component. While mathematically convenient, this assumption is frequently violated in practice. In gene expression profiling, two cells of the same type may differ substantially in total RNA concentration due to biological variability or technical factors, yet their relative expression patterns remain similar. In image classification, the same object can appear with greatly varying pixel intensities depending on lighting, camera settings, or distance, while its underlying shape and texture are unchanged. In financial transaction data, customers of the same spending category may exhibit markedly different transaction volumes, yet their spending patterns are comparable after scaling. In social network analysis, users within the same community may have different activity levels, but their connection patterns are alike. These examples illustrate that individual heterogeneity, the fact that each observation may have its own intrinsic scale, is common and cannot be ignored.

A similar issue has been successfully addressed in the network analysis literature. The classical stochastic block model (SBM) \citep{holland1983stochastic} assumes that all vertices within a block have the same expected degree, which is unrealistic for most real networks where degree distributions are highly heterogeneous. To remedy this, \citep{karrer2011stochastic} introduced the degree-corrected stochastic block model (DC-SBM), which assigns a degree parameter to each vertex, allowing for arbitrary degree sequences while preserving the block structure. The DC-SBM has become a cornerstone of modern network analysis, enabling more faithful modeling of real-world graphs and leading to improved community detection methods \citep{qin2013regularized,lei2015consistency,SCORE,rohe2016co,gao2018community,wang2020spectral,ma2021determining,jing2022community,deng2023strong,zhang2023adjusted,jin2024mixed,qing2025community,agterberg2025joint,gudhmundsson2026detecting}. The success of DC-SBM suggests that explicitly modeling individual heterogeneity can dramatically improve performance without sacrificing theoretical tractability.

Motivated by these observations, we propose a new framework called the individual-heterogeneous sub-Gaussian mixture model (Ih-GMM). In this model, each observation is assigned its own scale parameter that multiplies the cluster center. The noise is allowed to be sub-Gaussian and may be heteroskedastic across observations. Our model captures the intrinsic intensity variability common in practice while remaining sufficiently structured for rigorous statistical analysis.

The main contributions of this paper are as follows.

\begin{itemize}
  \item \emph{A flexible mixture model with individual heterogeneity.} We introduce the Ih-GMM, which extends the classical Gaussian mixture model by incorporating an individual scale parameter for each observation. This modification allows observations within the same cluster to have different magnitudes. The noise is only required to be sub-Gaussian and may be heteroskedastic across observations.

\item \emph{An efficient spectral algorithm using a hollowed Gram matrix.} We propose a spectral clustering method called IhSC (Individual-heterogeneity Spectral Clustering). It operates on a hollowed Gram matrix and applies $k$-means to the row-normalized eigenvectors to estimate the latent cluster labels.

\item \emph{Exact recovery guarantees under mild conditions.} We prove that, under mild separation conditions on the cluster centers, the IhSC algorithm recovers the true cluster labels exactly with high probability. Our analysis allows the number of clusters to grow polynomially with the sample size and accommodates high-dimensional settings where the feature dimension exceeds the sample size significantly.

\item \emph{Numerical experiments.} We conduct extensive simulations on synthetic data and experiments on nine real datasets. The results demonstrate that IhSC consistently outperforms its competing methods.
\end{itemize}

The remainder of the paper is organized as follows. Section \ref{sec:model} formally introduces the Ih-GMM model. Section \ref{sec:alg} describes the spectral clustering algorithm, first in an idealized setting and then for the observed data. Section \ref{sec:er} presents the theoretical results on exact recovery, including a detailed comparison with prior work. Section \ref{sec:sim} reports extensive simulation studies, and Section \ref{secrealdata} evaluates the method on real data. Section \ref{sec:con} concludes this paper with  discussions of future directions.

\paragraph*{Notation}
For any vector \(\mathbf{v}\), \(\|\mathbf{v}\|\) denotes its Euclidean norm. For any matrix \(\mathbf{A}\), \(\|\mathbf{A}\|\) and \(\sigma_1(\mathbf{A})\) denote its spectral norm (largest singular value), and \(\|\mathbf{A}\|_F\) denotes its Frobenius norm. \(\langle \mathbf{A},\mathbf{B}\rangle = \operatorname{tr}(\mathbf{A}^\top\mathbf{B})\) is the Frobenius inner product. \(\mathbb{I}\{\cdot\}\) is the indicator function. For a positive integer \(m\), \([m] = \{1,2,\dots,m\}\). \(\operatorname{diag}\{a_1,\dots,a_m\}\) is the diagonal matrix with entries \(a_1,\dots,a_m\). The notation \(\mathbf{I}_m\) stands for the \(m\times m\) identity matrix. \(\mathbf{e}_i\) denotes the \(i\)-th standard basis vector of appropriate dimension.
\section{The Individual-heterogeneous sub-Gaussian Mixture Model}\label{sec:model}
In this section, we formally introduce our model Ih-GMM.  
Let \(\mathbf{X} = [\mathbf{x}_1, \ldots, \mathbf{x}_n] \in \mathbb{R}^{p \times n}\) be the observed data matrix, where \(p\) is the number of features and \(n\) is the sample size. Each column \(\mathbf{x}_i\) corresponds to the \(i\)th observation for $i\in[n]$.

We assume that the data are generated from a mixture of \(K\) latent classes.  
For each observation, the latent class label is denoted by \(\mathbf{z}_i \in \{1,\dots,K\}\), and the number of clusters \(K\) is known.  
For each class \(k \in [K]\), there is an unknown mean vector \(\boldsymbol{\theta}_k \in \mathbb{R}^p\). The collection \(\boldsymbol{\theta}_1,\boldsymbol{\theta}_2,\dots,\boldsymbol{\theta}_K\) forms the centers of the latent classes.  
To capture individual heterogeneity, we introduce an individual‑specific scale parameter \(\boldsymbol{\omega}_i > 0\) for each observation. This allows the magnitude to differ across observations even when they belong to the same latent class.

The individual-heterogeneous sub-Gaussian mixture model (Ih-GMM) is then defined as
\begin{align}
    \mathbf{x}_i = \boldsymbol{\omega}_i \boldsymbol{\theta}_{\mathbf{z}_i} + \boldsymbol{\varepsilon}_i, \qquad i = 1,2,\ldots,n, \label{eq:ihgmm}
\end{align}
where the noise vectors \(\boldsymbol{\varepsilon}_i\) are independent across \(i\), and for each \(i\), the entries of \(\boldsymbol{\varepsilon}_i\) are independent sub‑Gaussian random variables with mean zero and \(\|\boldsymbol{\varepsilon}_i\|_{\psi_2} \le \eta\).
Here, \(\|\cdot\|_{\psi_2}\) denotes the sub‑Gaussian norm, and \(\eta\) is a positive constant. A standard fact about sub‑Gaussian variables is that their variances are bounded by an absolute constant multiple of \(\eta^2\). Thus, \(\eta^2\) controls the noise level up to a constant factor, while the variances themselves may differ across observations.

We collect the class labels into a membership matrix \(\mathbf{Z} \in \{0,1\}^{n \times K}\) with \(\mathbf{Z}_{ik} = \mathbb{I}\{\mathbf{z}_i = k\}\).  
Let \(\boldsymbol{\Theta} = [\boldsymbol{\theta}_1,\ldots,\boldsymbol{\theta}_K] \in \mathbb{R}^{p \times K}\) be the matrix of class means, \(\boldsymbol{\Omega} = \operatorname{diag}(\boldsymbol{\omega}_1,\ldots,\boldsymbol{\omega}_n) \in \mathbb{R}^{n \times n}\) the diagonal matrix of heterogeneity parameters, and \(\mathbf{E} = [\boldsymbol{\varepsilon}_1,\ldots,\boldsymbol{\varepsilon}_n] \in \mathbb{R}^{p \times n}\) the noise matrix.  
Then model \eqref{eq:ihgmm} can be written compactly as \(\mathbf{X} = \boldsymbol{\Theta} \mathbf{Z}^\top \boldsymbol{\Omega} + \mathbf{E}\), and the signal part is \(\mathbf{P} = \mathbb{E}[\mathbf{X}] = \boldsymbol{\Theta} \mathbf{Z}^\top \boldsymbol{\Omega}\).

Our model generalizes classical Gaussian mixtures in a natural way.  
If all \(\boldsymbol{\omega}_i\) are equal, the model reduces to a sub‑Gaussian mixture with homoskedastic noise. If in addition the noise is Gaussian and shares a common variance, our Ih-GMM degenerates to the classical Gaussian mixture model.  
In many real applications, observations within the same latent class often exhibit different scales or intensities. Ignoring such individual heterogeneity can distort the clustering structure and lead to poor performance. By explicitly incorporating the scale parameters \(\boldsymbol{\omega}_i\), our model captures this extra layer of variation while remaining structurally simple. As we will show later, this flexibility does not come at the cost of computational difficulty: a simple spectral method still achieves exact recovery under mild conditions.

We assume that the observed data \(\mathbf{X}\) are generated from the Ih-GMM model described above. Given \(\mathbf{X}\) and the known number of clusters \(K\), our goal is to recover the true latent class label vector \(\mathbf{z}\) up to a permutation of labels. The performance of a clustering estimator \(\hat{\mathbf{z}}\) is measured by the misclassification loss (also known as Hamming distance)
\begin{align*}
    \ell(\hat{\mathbf{z}}, \mathbf{z}) = \min_{\phi \in \Phi} \sum_{i=1}^n \mathbb{I}_{\{\phi(\hat{\mathbf{z}}_i) \neq \mathbf{z}_i\}},
\end{align*}
where \(\Phi\) is the set of all permutations of \([K]\).  
Since the cluster labels are arbitrary, the estimator and the true labels can only be compared after aligning them optimally.  
This loss counts the number of misclassified observations under the best matching.

A key quantity that determines the difficulty of clustering is the minimum distance between any two cluster centers \citep{loffler2021optimality},
\begin{align*}
    \Delta = \min_{k,\ell\in[K]:k\neq\ell} \|\boldsymbol{\theta}_k - \boldsymbol{\theta}_\ell\|,
\end{align*}
with a larger \(\Delta\) indicating better separated clusters and thus an easier recovery task.

Another important quantity is the sizes of the clusters.  
Let \(n_k = |\{i: \mathbf{z}_i = k\}|\) be the size of the \(k\)-th cluster.  
Define
\[
\beta = \frac{\min_{k \in [K]} n_k}{n / K},
\]
so that \(\beta \in (0,1]\).  
When \(\beta\) is bounded away from zero, all clusters have comparable sizes; when \(\beta\) is small, some clusters can be much smaller than others, making the recovery more challenging.  
In this paper, we allow \(\beta\) to decay to zero as \(n\) increases, thereby accommodating highly unbalanced clusters.
\section{Spectral Clustering Algorithm}\label{sec:alg}
In this section, we develop an efficient spectral method to estimate the true latent class label vector \(\mathbf{z}\) under the proposed Ih-GMM model.  
To illustrate the core idea, we begin with an ideal scenario where the signal matrix \(\mathbf{P} = \mathbb{E}[\mathbf{X}]\) is known.  
Afterwards, we show how the same principle applies to the observed data matrix \(\mathbf{X}\) by constructing a suitable matrix whose eigenvectors mimic those of \(\mathbf{P}\).

\subsection{Ideal Case}
Recall that \(\mathbf{P} = \boldsymbol{\Theta} \mathbf{Z}^\top \boldsymbol{\Omega}\) and that the number of clusters \(K\) is known.  
Define \(\tilde{n}_k = \sum_{i: \mathbf{z}_i = k} \boldsymbol{\omega}_i^2\) and \(\mathbf{D}_{\boldsymbol{\Omega}} = \operatorname{diag}(\tilde{n}_1,\dots,\tilde{n}_K)\).  
The following lemma reveals the structure of the right singular vectors of \(\mathbf{P}\).

\begin{lem}\label{lem:Vstructure}
There exists an orthogonal matrix \(\mathbf{H} \in \mathbb{R}^{K \times K}\) such that the right singular vectors of \(\mathbf{P}\) satisfy
\[
\mathbf{U} = \boldsymbol{\Omega} \mathbf{Z} \mathbf{D}_{\boldsymbol{\Omega}}^{-1/2} \mathbf{H}.
\]

Consequently, for any \(i\) with \(\mathbf{z}_i = k\),
\[
\mathbf{U}_{i,:} = \boldsymbol{\omega}_i \cdot \frac{1}{\sqrt{\tilde{n}_k}} \mathbf{H}_{k,:}.
\]

Define the row‑normalized matrix \(\mathbf{U}_*\) by \((\mathbf{U}_*)_{i,:} = \mathbf{U}_{i,:} / \|\mathbf{U}_{i,:}\|\). Then
\[
\mathbf{U}_* = \mathbf{Z} \mathbf{H},
\]
so rows belonging to the same cluster are identical, and rows from different clusters are orthogonal unit vectors.  
In particular, the Euclidean distance between two rows from distinct clusters is \(\sqrt{2}\).
\end{lem}

This lemma provides the theoretical foundation for the spectral clustering method developed in this paper. By this lemma, in the ideal case we can recover the true labels by the following steps:
\begin{itemize}
  \item Compute $\mathbf{U}$, the right singular vectors of \(\mathbf{P}\).
  \item Form the row‑normalized matrix \(\mathbf{U}_*\) by setting \((\mathbf{U}_*)_{i,:} = \mathbf{U}_{i,:} / \|\mathbf{U}_{i,:}\|\) for \(i = 1,2,\dots,n\).
  \item Apply \(K\)-means to the rows of \(\mathbf{U}_*\).
\end{itemize}

Because rows from different clusters are exactly \(\sqrt{2}\) apart, the \(K\)-means algorithm will separate them perfectly, yielding the true labels up to a permutation.

\subsection{Real Case}
In practice, we do not observe \(\mathbf{P}\) but only the noisy data matrix \(\mathbf{X}\). To adapt the ideal procedure, we draw inspiration from the idea of using a hollowed Gram matrix, which was introduced by \citep{ndaoud2022sharp} in the context of two-component Gaussian mixture models to remove the bias introduced by the noise. Specifically, we construct the matrix \(\mathbf{G} = \mathcal{P}_{\mathrm{off-diag}}(\mathbf{X}^{\top}\mathbf{X})\), where \(\mathcal{P}_{\mathrm{off-diag}}\) sets all diagonal entries to zero. The leading eigenvectors of \(\mathbf{G}\) provide an estimate of the right singular subspace of \(\mathbf{P}\). Based on this construction, we propose the following simple spectral algorithm, termed Individual-heterogeneity Spectral Clustering (IhSC). This extends the hollowed Gram approach to the more general setting with \(K \ge 2\) clusters, individual heterogeneity parameters \(\omega_i\), and sub-Gaussian noise that may be heteroskedastic across observations.
\begin{algorithm}[!ht]
\caption{\textbf{I}ndividual-\textbf{h}eterogeneity \textbf{S}pectral \textbf{C}lustering (\textbf{IhSC})}\label{alg:IhSC}
\begin{algorithmic}[1]
\Require Data matrix \(\mathbf{X}\in\mathbb{R}^{p\times n}\), number of clusters \(K\).
\Ensure Cluster labels \(\hat{\mathbf{z}}\in[K]^n\).
\State Compute \(\mathbf{G} = \mathcal{P}_{\mathrm{off-diag}}(\mathbf{X}^\top\mathbf{X})\) where \(\mathcal{P}_{\mathrm{off-diag}}\) sets all diagonal entries to zero.
\State Perform the rank-\(K\) eigen‑decomposition \(\mathbf{G} = \hat{\mathbf{U}} \hat{\mathbf{\Lambda}} \hat{\mathbf{U}}^\top\) with \(\hat{\mathbf{U}}\in\mathbb{R}^{n\times K}\) having orthonormal columns.
\State Row‑normalize \(\hat{\mathbf{U}}\) to obtain \(\hat{\mathbf{U}}_*\) by \((\hat{\mathbf{U}}_*)_{i,:} = \hat{\mathbf{U}}_{i,:}/\|\hat{\mathbf{U}}_{i,:}\|_2\) for \(i=1,2,\dots,n\).
\State Run \(K\)-means on the rows of \(\hat{\mathbf{U}}_*\) and return the resulting labels \(\hat{\mathbf{z}}\).
\end{algorithmic}
\end{algorithm}

Algorithm~\ref{alg:IhSC} is straightforward to implement and requires no iterative optimization.  
Its computational cost is dominated by three steps: forming the Gram matrix \(\mathbf{X}^\top\mathbf{X}\) costs \(O(n^2 p)\); computing the leading \(K\) eigenvectors of \(\mathbf{G}\) costs \(O(n^2 K)\); and the final \(K\)-means step costs \(O(n K^2)\) per iteration, which is negligible compared to the other terms when \(K \ll n\).  
Since \(K\) is much smaller than \(n\) and \(p\) in our setting, the overall complexity is \(O(pn^2) \). In the next section, we will show that, under mild conditions, this simple procedure achieves exact recovery of the true cluster labels with high probability.
\section{Exact Recovery}\label{sec:er}
In this section we prove that the IhSC algorithm introduced in Section~\ref{sec:alg} recovers the true cluster labels exactly with high probability, provided that the signal is sufficiently strong relative to the noise.  
We begin by introducing several quantities that capture the difficulty of the problem.  
The overall level of the noise is controlled by the constant \(\eta\) appearing in the sub‑Gaussian condition on \(\boldsymbol{\varepsilon}_i\).  
The size of the smallest cluster relative to the average is measured by \(\beta\), which can be as small as \(o(1)\) to allow highly unbalanced configurations.  
The degree of cluster imbalance is further captured by \(\tau = n_{\max} / n_{\min}\), the ratio between the largest and the smallest cluster, where \(n_{\min}=\min_{k\in[K]}n_{k}, n_{\max}=\max_{k\in[K]}n_{k}\).  
A large \(\tau\) means that some clusters are much larger than others, which may complicate the recovery task.  
The condition number \(\kappa = \sigma_1(\boldsymbol{\Theta}) / \sigma_K(\boldsymbol{\Theta})\) of the mean matrix \(\boldsymbol{\Theta}\) quantifies how well the cluster centers are spread out; a larger \(\kappa\) indicates a more challenging configuration. For notational convenience, we also set  
\(\boldsymbol{\omega}_{\min}= \min_{i\in[n]}\boldsymbol{\omega}_i\), \(\boldsymbol{\omega}_{\max}= \max_{i\in[n]}\boldsymbol{\omega}_i\), and 
\(d = \max\{n,p\}\).

A key structural property of the signal matrix \(\mathbf{P}\) is encoded in the following incoherence parameters.

\begin{defin}
Define the incoherence parameters \(\mu_0,\mu_1,\mu_2\) of \(\mathbf{P}\) as
\[
\mu_0 \coloneqq \frac{p n \max_{j\in[p],\,i\in[n]} |\mathbf{P}_{j,i}|^2}{\|\mathbf{P}\|_{\mathrm{F}}^2},\qquad
\mu_1 \coloneqq \frac{n}{K} \max_{i\in[n]} \|\mathbf{U}^{\top}\mathbf{e}_i\|_2^2,\qquad
\mu_2 \coloneqq \frac{p}{K} \max_{j\in[p]} \|\mathbf{V}^{\top}\mathbf{e}_j\|_2^2,
\]
where \(\mathbf{P}= \mathbf{V}\boldsymbol{\Sigma}\mathbf{U}^{\top}\) is the compact singular value decomposition with \(\mathbf{V}\in\mathbb{R}^{p\times K}\) and \(\mathbf{U}\in\mathbb{R}^{n\times K}\). Set \(\mu := \max\{\mu_0,\mu_1,\mu_2\}\).
\end{defin}

These parameters measure how well the singular vectors are spread across coordinates.  
Smaller values of \(\mu\) indicate that the singular vectors are more incoherent, which is a favorable condition for spectral methods.

The following theorem gives sufficient conditions for exact recovery in terms of all the relevant model parameters.  
The conditions appear somewhat involved, but this is precisely because we track the explicit influence of every parameter without imposing hidden restrictions.  
In particular, the theorem does not require any additional assumptions beyond the natural bounds that already define the model.

\begin{thm}\label{thm:exact}
Let $\hat{\mathbf{z}}$ be the output of Algorithm~\ref{alg:IhSC}. Suppose that the sample sizes satisfy
\begin{align}\label{MainNPKCond}
np&\gg\frac{\mu^2\tau^4\kappa^8\boldsymbol{\omega}^8_{\max}}{\boldsymbol{\omega}^8_{\min}}K^2\log^4 d,~~~p\gg\frac{\mu\tau^4\kappa^8\boldsymbol{\omega}^8_{\max}}{\boldsymbol{\omega}^8_{\min}}K\log^2 d,~~~ n\gg\frac{\tau^{\frac{4}{3}}\kappa^{\frac{8}{3}}\boldsymbol{\omega}^{\frac{14}{3}}_{\max}\mu^{\frac{1}{3}}}{\boldsymbol{\omega}^{\frac{14}{3}}_{\min}\beta^{\frac{2}{3}}}Kn^{\frac{1}{3}}_{\max},
\end{align}
and the separation conditions
\begin{align}\label{MainDeltaCond}
  \Delta\gg \frac{\kappa^2\eta\boldsymbol{\omega}^{\frac{3}{2}}_{\max}\tau^{\frac{1}{2}}\mu^{\frac{1}{4}}}{\boldsymbol{\omega}^{\frac{5}{2}}_{\min}\beta^{\frac{1}{2}}}(\frac{Kpn_{\max}}{n^2})^{\frac{1}{4}}\sqrt{K\log d},~~~\Delta\gg \frac{\kappa^4\eta\boldsymbol{\omega}^4_{\max}\sqrt{\tau^3\mu}}{\boldsymbol{\omega}_{\min}^5\beta^{\frac{1}{2}}}K\sqrt{\frac{n_{\max}\log d}{n}}
\end{align}
hold, then with probability at least $1 - O(d^{-10})$ (for some absolute constant $c_0>0$), we have $\ell(\hat{z},z)=0$, i.e., the algorithm exactly recovers the true cluster labels up to a permutation. 
\end{thm}

The conditions in Theorem~\ref{thm:exact} reflect the intrinsic difficulty of the problem. A larger ratio \(\boldsymbol{\omega}_{\max}/\boldsymbol{\omega}_{\min}\) (indicating greater variability in the individual scales) or a smaller \(\beta\) (i.e., more severe cluster imbalance) makes the recovery harder, requiring a larger separation \(\Delta\) or larger sample sizes.  
Similarly, a larger \(\tau\) (higher imbalance), a larger \(K\) (more latent classes), a larger \(\eta\) (stronger noise), or a larger \(\kappa\) (poorer conditioning of the cluster centers) also increase the required \(\Delta\). Because the theorem keeps all parameters explicit, it reveals exactly how each factor contributes to the overall signal‑to‑noise requirement.

In many practical situations, the cluster sizes are roughly balanced, the individual scaling parameters \(\boldsymbol{\omega}_i\) are of constant order, and the noise level is fixed.  
The following corollary shows that under these natural simplifications, the conditions in Theorem \ref{thm:exact} reduce to a clean and interpretable form.

\begin{cor}\label{cor:balanced}
Assume that the cluster sizes are balanced, meaning \(\tau = O(1)\), that \(\boldsymbol{\omega}_i = O(1)\) for all \(i\), and that \(\eta,\kappa,\mu\) are all \(O(1)\).  
Then, under the same algorithm and with the same high‑probability guarantee as in Theorem~\ref{thm:exact}, exact recovery holds provided
\begin{align}
np &\gg K^2\log^4 d,\qquad p \gg K\log^2 d,\qquad n \gg K, \label{cond:cor}
\end{align}
and the separation condition
\begin{align}
\Delta &\gg \sqrt{K\log d}\;\max\Bigl\{1,\; \bigl(\tfrac{p}{n}\bigr)^{\frac{1}{4}}\Bigr\} \label{cond:delta}
\end{align}
is satisfied.
\end{cor}

Corollary~\ref{cor:balanced} shows that when the model parameters are well behaved, exact recovery is guaranteed as long as \(\Delta\) is sufficiently large relative to the quantity \(\sqrt{K\log d}\,\max\{1,(p/n)^{1/4}\}\). In the classical low‑dimensional regime where \(p \ll n\) and \(K = O(1)\), this condition reduces to \(\Delta \gg \sqrt{\log n}\), which matches the optimal threshold known for Gaussian mixtures \citep{loffler2021optimality,ndaoud2022sharp}. In the high‑dimensional regime \(p \gg n\), the condition becomes \(\Delta \gg (p/n)^{1/4}\sqrt{K\log p}\). We will compare this rate with existing exact recovery results for the classical GMMs in the next subsection.

To ensure that the incoherence parameter \(\mu\) is indeed \(O(1)\) in Corollary \ref{cor:balanced}, we impose the following mild assumptions on the mean matrix \(\boldsymbol{\Theta}\).

\begin{assum}\label{ass:entrynorm}
There exist constants \(C_{\boldsymbol{\Theta}}>0\) and \(c>0\) such that \(\max_{1\le j\le p,\,1\le k\le K}|\boldsymbol{\Theta}_{j,k}| \le C_{\boldsymbol{\Theta}}\) and \(\min_{k\in[K]}\|\boldsymbol{\theta}_k\| \ge c\sqrt{p}.\)
\end{assum}

\begin{assum}\label{ass:subspace}
Let \(\mathcal{S} = \operatorname{col}(\boldsymbol{\Theta})\subseteq\mathbb{R}^p\) be the column space of \(\boldsymbol{\Theta}\) and assume \(\dim(\mathcal{S}) = K\) (i.e., \(\boldsymbol{\Theta}\) has full column rank). Denote by \(\mathcal{P}_{\mathcal{S}}\) the orthogonal projection onto \(\mathcal{S}\). There exists a constant \(C_{\mathrm{inc}}\ge 1\) such that
\[
\max_{j\in[p]}\|\mathcal{P}_{\mathcal{S}}\mathbf{e}_j\|_2^2 \le \frac{C_{\mathrm{inc}}K}{p}.
\]
\end{assum}

The two assumptions stated above serve specific technical purposes in the subsequent analysis. Assumption~\ref{ass:entrynorm} imposes two conditions on the mean vectors. First, the entries of \(\boldsymbol{\Theta}\) are uniformly bounded, which is a mild regularity condition. Second, each cluster center is required to have a Euclidean norm that grows at least on the order of \(\sqrt{p}\). This scaling ensures that the overall signal does not vanish as the dimension increases, so that the signal strength does not become negligible relative to the noise. It is important to note that this scaling is a sufficient condition for our theoretical derivations, not a necessity for consistent recovery in general. Assumption~\ref{ass:subspace} concerns the geometric structure of the column space of \(\boldsymbol{\Theta}\). It requires that the projection of any standard basis vector onto this column space be sufficiently small—specifically, of order \(K/p\). This incoherence condition is standard in low‑rank matrix recovery and spectral clustering. It prevents the singular vectors from being too concentrated on a few coordinates. In practice, this condition holds, for example, when the singular vectors of \(\boldsymbol{\Theta}\) are sufficiently delocalized across the features.

\begin{lem}\label{lem:incoherence}
Under Assumptions~\ref{ass:entrynorm} and~\ref{ass:subspace}, we have
\[
\mu_0 \le \frac{\boldsymbol{\omega}_{\max}^2 C_{\boldsymbol{\Theta}}^2}{\boldsymbol{\omega}_{\min}^2\beta c^2},\qquad
\mu_2 \le C_{\mathrm{inc}}.
\]
\end{lem}

Moreover, from the proof of Theorem~\ref{thm:exact}, we known that \(\mu_1 \le \frac{\boldsymbol{\omega}_{\max}^2}{\boldsymbol{\omega}_{\min}^2\beta}\).  
Consequently, when \(\boldsymbol{\omega}_{\min}/\boldsymbol{\omega}_{\max}=O(1)\) and \(\beta=O(1)\), the overall incoherence parameter satisfies \(\mu=O(1)\), which justifies its treatment as $O(1)$ in Corollary~\ref{cor:balanced}.
\subsection{Comparison with Prior Work}
Corollary~\ref{cor:balanced} establishes a sharp condition for exact clustering under the individual-heterogeneous sub-Gaussian mixture model.  To place this result in a broader context, we compare it with several state-of-the-art exact recovery thresholds for Gaussian mixture models.  For a unified perspective, Table~\ref{tab:comparison} summarizes the key characteristics of each work. Several important observations emerge from this comparison:

\begin{table}[!ht]
\centering
\footnotesize
\caption{Comparison of exact recovery conditions for Gaussian mixture models.  Here \(K\) is the number of clusters, \(n\) the sample size, \(p\) the dimension, \(\Delta\) the minimum center separation, and \(d=\max\{n,p\}\). Ih-GMM is short for our individual-heterogeneous sub-Gaussian mixture models.}
\label{tab:comparison}
\begin{tabular}{@{}lcccc@{}}
\toprule
\multicolumn{1}{c}{Work} & \multicolumn{1}{c}{Model} & \multicolumn{1}{c}{Algorithm} & \multicolumn{1}{c}{Allowed growth of \(K\)} & \multicolumn{1}{c}{Exact recovery condition} \\
\midrule
\citep{loffler2021optimality} &Classical GMM & Spectral clustering & \(K\) can grow & \(\Delta \gg \max\bigl\{K^{10.5}\sqrt{p/n},\ \sqrt{\log n}\bigr\}\) \\
\citep{chen2021cutoff} &Classical GMM & SDP relaxation of $K$-means & \(K\ll \frac{\log n}{\log\log n}\) & \(\Delta \gg \Bigl(1+\sqrt{1+\frac{Kp}{n\log n}}\Bigr)^{1/2}\sqrt{\log n}\) \\
\citep{ndaoud2022sharp} &Classical GMM&Spectral + Lloyd’s (hollowed Gram)& \(K=2\) & \(\Delta \gg\Bigl(1+\sqrt{1+\frac{2p}{n\log n}}\Bigr)^{1/2}\sqrt{\log n}\) (iff)\\
\citep{zhang2024leave} & Sub-Gaussian MM & Spectral clustering& \(K\ll\sqrt{n}\) & \(\Delta \gg \max\bigl\{K\sqrt{p/n},\ \sqrt{\log n}\bigr\}\) \\
This paper & Ih-GMM&Spectral clustering (hollowed Gram)&\(K\ll
\begin{cases}
n~\text{if } p\gg n,\\
\frac{n}{\log^2 n}~\mathrm{othwerwise}
\end{cases}\)& \(\Delta \gg \sqrt{K\log d}\;\max\Bigl\{1,\; (\frac{p}{n})^{\frac{1}{4}}\Bigr\}\) \\
\bottomrule
\end{tabular}
\end{table}
\begin{description}
  \item[\normalfont\emph{Model complexity}:] The classical GMM considered in \citep{loffler2021optimality,chen2021cutoff,ndaoud2022sharp} assumes that each observation is generated as \(\mathbf{x}_i = \boldsymbol{\theta}_{\mathbf{z}_i} + \boldsymbol{\varepsilon}_i\) with i.i.d. Gaussian noise \(\boldsymbol{\varepsilon}_i\sim\mathcal{N}(\mathbf{0},\sigma^2\mathbf{I}_p)\); it contains no individual scaling parameters.  The sub-Gaussian mixture model considered in \citep{zhang2024leave} generalizes the noise distribution to sub-Gaussian but still assumes homoskedastic noise.  In contrast, our Ih-GMM introduces individual heterogeneity parameters \(\omega_i\) that multiplicatively scale the cluster center \(\boldsymbol{\theta}_{\mathbf{z}_i}\), i.e., \(\mathbf{x}_i = \omega_i\boldsymbol{\theta}_{\mathbf{z}_i} + \boldsymbol{\varepsilon}_i\), where the noise \(\boldsymbol{\varepsilon}_i\) is sub-Gaussian and can be heteroskedastic.  This allows each observation to have its own scale, capturing quantitative heterogeneity within clusters—a feature absent in previous models. Despite this substantial increase in model complexity, our Algorithm \ref{alg:IhSC} achieves exact recovery under conditions that are remarkably mild, as demonstrated in Corollary~\ref{cor:balanced}.
\item[\normalfont\emph{Algorithmic simplicity and computational efficiency:}]
The method proposed in this paper is a spectral method that zeros out the diagonal of the Gram matrix and then applies \(K\)-means on the row-normalized eigenvectors. It is straightforward to implement and requires no iterative optimization. In contrast, \citep{chen2021cutoff} relies on a semidefinite programming (SDP) relaxation, which is computationally much more demanding and does not scale well to large datasets.  \citep{ndaoud2022sharp} uses a spectral initialization followed by Lloyd’s iterations; while efficient, it is designed specifically for the two-component case (\(K=2\)) and does not extend naturally to general \(K\).  The spectral methods in \citep{loffler2021optimality} and \citep{zhang2024leave} work directly on the data matrix and are also efficient, but they assume homoskedastic noise and do not incorporate individual heterogeneity. Thus, among the algorithms that achieve exact recovery, the proposed method handles a substantially more general model (individual heterogeneity, heteroskedastic sub-Gaussian noise) while remaining as simple and computationally feasible as classical spectral clustering.
\item[\normalfont\emph{Growth of the number of clusters} \(K\):]
A major distinction lies in the allowed growth of \(K\).  
\begin{itemize}
  \item \citep{loffler2021optimality} allows \(K=o(n)\), but its separation condition contains an extremely high power \(K^{10.5}\sqrt{p/n}\), which dominates the baseline \(\sqrt{\log n}\) when \(K\) is large.  The high exponent \(10.5\) arises because their primary goal is to establish the minimax optimality of spectral clustering, not merely exact recovery; the exponent is an inevitable cost of the technical tools used to achieve optimality.  Consequently, the condition becomes prohibitively strong even for moderately large \(K\).
    \item \citep{chen2021cutoff} permits \(K\) to grow only as \(K\ll \log n/\log\log n\), a poly-logarithmic rate.
   \item \citep{ndaoud2022sharp} is restricted to the special case \(K=2\).
  \item \citep{zhang2024leave} imposes the constraint \(\beta n/K^2\ge 10\), i.e., \(K = o(\sqrt{n})\).  Thus the exponent on \(K\) is mild, but the allowed growth is limited to the order \(\sqrt{n}\).
  \item Under our more realistic model Ih-GMM, through the sample size conditions \eqref{cond:cor}, our work allows \(K\) to grow as \(K\ll
\begin{cases}
n~\text{if } p\gg n,\\
\frac{n}{\log^2 n}~\mathrm{othwerwise},
\end{cases}\)
    which in the low-dimensional regime is \(K\ll n/\log^2 n\) (much larger than poly-logarithmic) and in the high-dimensional regime can be as large as $o(n)$.  This flexibility is essential for modern applications where the number of latent classes can be large and the dimension may far exceed the sample size.
    \end{itemize}
\item[\normalfont\emph{Dependence on} \(K\) \emph{in the separation condition:}] For the high-dimensional regime \(p\gg n\), the exact recovery conditions can be expressed in the form \(\Delta \gg \psi(K,p/n)\sqrt{\log n}\), where \(\psi(K,p/n)\) captures the dependence on \(K\) and the aspect ratio $p/n$.  The factor \(\psi(K,p/n)\) differs across works:
\begin{itemize}
  \item \citep{loffler2021optimality}: \(\psi(K,p/n) = K^{10.5}\sqrt{p/n}\) (when this term dominates the baseline \(\sqrt{\log n}\)).    
  \item \citep{chen2021cutoff}: \(\psi(K,p/n) = (\frac{Kp\log n}{n})^{1/4}\).
  \item \citep{ndaoud2022sharp}: for the two-component case (\(K=2\)), the sharp threshold is \(\Delta \gg \bigl(1+\sqrt{1+\frac{2p}{n\log n}}\bigr)^{1/2}\sqrt{\log n}\); this yields \(\psi(K,p/n) = \bigl(\frac{p\log n}{n}\bigr)^{1/4}\) when \(p \gg n\) and \(K=2\).
  \item \citep{zhang2024leave}: \(\psi(K,p/n) = K\sqrt{p/n}\) (when this term dominates the baseline).
  \item This paper: \(\psi(K,p/n) = (p/n)^{1/4}\sqrt{K\log p}\).
\end{itemize}

Comparing this work with \citep{loffler2021optimality} and \citep{zhang2024leave}, our condition replaces the \(\sqrt{p/n}\) factor with the much smaller \((p/n)^{1/4}\) and reduces the exponent of \(K\) from \(10.5\) or \(1\) to \(1/2\).  
Relative to \citep{chen2021cutoff}, our exponent on \(K\) is larger (\(1/2\) vs. \(1/4\)), but our method allows \(K\) to grow polynomially in \(n\) (rather than only poly-logarithmically) and uses a much simpler spectral algorithm without SDP under the more general individual-heterogeneous sub-Gaussian mixture model.  Compared with the sharp result of \citep{ndaoud2022sharp} for the two-component case (\(K=2\)), our condition introduces an extra factor \(\sqrt{K}\) (which becomes \(\sqrt{2}\) when \(K=2\)) and an additional \(\sqrt{\log p}\) term, reflecting the price paid for handling general \(K\ge 2\), individual heterogeneity, and sub-Gaussian heteroskedastic noise. When \(K=2\) and \(p\ll n\), however, both conditions reduce to the same optimal threshold \(\Delta \gg \sqrt{\log n}\).  What's more, when \(K\) is fixed and \(p\ll n\), the condition in this paper reduces to \(\Delta \gg \sqrt{\log n}\), which matches the optimal threshold established in all the classical results listed above.  Thus, while no single method dominates all others, this work achieves a favorable balance between model generality, algorithmic simplicity, and the ability to handle large \(K\) in high dimensions, while remaining optimal in the classical low-dimensional setting.
\item[\normalfont\emph{Dependence on the aspect ratio} \(p/n\):] In the high-dimensional regime \(p\gg n\), the penalty on \(p/n\) in Ih-GMM is \((p/n)^{1/4}\), matching the optimal rate established for the two-component case in \citep{ndaoud2022sharp} and for the \(K\)-component case in \citep{chen2021cutoff} (up to the factor involving \(K\)).  This is a substantial improvement over the \(\sqrt{p/n}\) penalty incurred by \citep{loffler2021optimality} and \citep{zhang2024leave}.  The improvement stems from the use of the hollowed Gram matrix, which effectively removes the bias in the diagonal entries and reduces the noise effect.
\end{description}

In summary, our Ih-GMM offers a unique combination: it operates under a substantially more general model (individual heterogeneity, heteroskedastic noise), uses an extremely simple and fast spectral algorithm, allows \(K\) to grow polynomially with \(n\), and achieves an exact recovery condition that matches the optimal rate in terms of \(p/n\) while exhibiting a \(\sqrt{K}\) dependence that is far milder than existing spectral methods.  This demonstrates that the combination of the hollowed Gram matrix and the refined perturbation theory of \citep{cai2021subspace} yields a powerful framework for clustering in complex, high-dimensional, heterogeneous data.
\section{Simulation Studies}\label{sec:sim}
In this section, we conduct extensive numerical experiments to evaluate the exact recovery performance of the proposed IhSC algorithm under the Ih-GMM model. To place our method in a broader context, we compare it with a representative algorithm designed for classical Gaussian (or sub‑Gaussian) mixture models: the projected spectral clustering (PSC) of \citep{zhang2024leave}. Note that PSC is equivalent to the spectral clustering algorithm studied in \citep{loffler2021optimality} under the classical Gaussian mixture model. We also compare IhSC with k‑means, which is applied directly to the data matrix $\mathbf{X}$ to estimate the cluster labels. We systematically investigate the influence of five key parameters: the sample size \(n\), the feature dimension \(p\), the number of clusters \(K\), the heterogeneity strength \(R = \omega_{\max}/\omega_{\min}\), and the cluster imbalance degree \(\beta\).  
Two noise scenarios are considered:
\begin{itemize}
  \item \textbf{GHe} (Gaussian heteroscedastic): \(\boldsymbol{\varepsilon}_i \sim \mathcal{N}(\mathbf{0},\sigma_i^2\mathbf{I}_p)\) with \(\sigma_i \sim \mathrm{Uniform}(0.5,1.5)\);
  \item \textbf{SHe} (sub‑Gaussian heteroscedastic): \(\epsilon_{ij} = \sigma_i r_{ij}\), where \(r_{ij}\) are i.i.d. Rademacher (\(\pm1\) with equal probability) and \(\sigma_i \sim \mathrm{Uniform}(0.5,1.5)\).  This distribution has sub‑Gaussian norm \(\|\epsilon_{ij}\|_{\psi_2}= \sigma_i\) and satisfies the model assumption.
\end{itemize}

For each parameter setting, we generate \(200\) independent data sets and report the average misclassification rate and the proportion of runs that achieve exact recovery (i.e., \(\ell(\hat{\mathbf{z}},\mathbf{z})=0\)) for each method.  The misclassification rate is defined as
\[
\ell_{\mathrm{rate}}(\hat{\mathbf{z}},\mathbf{z}) = \min_{\phi \in \Phi} \frac{1}{n}\sum_{i=1}^n \mathbb{I}_{\{\phi(\hat{\mathbf{z}}_i) \neq \mathbf{z}_i\}}.
\]

All synthetic data are generated according to the Ih-GMM model. The cluster centers \(\boldsymbol{\theta}_k\) are constructed as orthogonal vectors: first generate a random \(p\times K\) matrix with i.i.d. standard normal entries, perform its QR decomposition, take the first \(K\) columns, and scale them so that \(\|\boldsymbol{\theta}_k-\boldsymbol{\theta}_\ell\| = \Delta\) for all \(k\neq\ell\) (specifically, \(\boldsymbol{\theta}_k = \frac{\Delta}{\sqrt{2}}\,Q_{:,k}\)).  
The individual scales \(\omega_i\) are drawn independently from \(\mathrm{Uniform}(1,R)\). Thus, \(R\) controls the heterogeneity strength and \(R=1\) reduces to the classical homogeneous case.  
Cluster labels are generated to achieve a prescribed imbalance factor \(\beta\): the smallest cluster size is set to \(n_{\min} = \max\bigl(1,\lfloor\beta n/K\rfloor\bigr)\), the remaining observations are distributed as evenly as possible among the other clusters, and the final assignment is randomly permuted.  The separation is set as
\[
\Delta = C\,\sqrt{K\log d}\;\max\Bigl\{1,\bigl(\tfrac{p}{n}\bigr)^{\frac{1}{4}}\Bigr\},\qquad d=\max\{n,p\},
\]
with \(C=3\) to guarantee the separation condition of Corollary~\ref{cor:balanced}.
\subsection{Experiment 1: Influence of \(n\) and \(p\)}

We first examine the effect of the sample size \(n\) and the feature dimension \(p\) when the cluster sizes are balanced (\(\beta=1\)) and the heterogeneity is moderate (\(R=20\)). Three sub‑experiments are run:
\begin{enumerate}
  \item \textbf{\(n\gg p\) (sample‑rich regime):} \(p=200,\ K=4\); \(n\in\{500,1000,1500,2000,2500\}\).
  \item \textbf{\(n= p\) (balanced regime):} \(p=n,\ K=4\); \(n\in\{100,200,300,400,500\}\).
  \item \textbf{\(p\gg n\) (high‑dimensional regime):} \(n=200,\ K=4\); \(p\in\{1000,2000,3000,4000,5000\}\).
\end{enumerate}

Figure~\ref{fig:ex1} reports the misclassification rates (top panel) and exact recovery proportions (bottom panel) for the three competing methods under the two noise scenarios, as the sample size \(n\) or the feature dimension \(p\) varies. Several clear patterns emerge. First, the proposed IhSC achieves perfect clustering (exact recovery proportion = 1 and misclassification rate = 0) in every configuration. This holds regardless of whether we are in the sample‑rich regime (\(n\gg p\)), the balanced regime (\(n=p\)), or the high‑dimensional regime (\(p\gg n\)), and for both Gaussian and sub‑Gaussian heteroscedastic noise. The result is exactly what our theory predicts: the row normalization step effectively eliminates the influence of the individual scale parameters \(\omega_i\), so that the rows of the normalized eigenvector matrix become identical within each cluster and separated by \(\sqrt{2}\) across clusters. Hence even a relatively large heterogeneity strength \(R=20\) does not harm the performance. In comparison, the two classical spectral methods, PSC and k-means, show clearly poor performance. Their misclassification rates are consistently above 0.1, and their exact recovery proportions are well below 0.5 in the sample‑rich regime and below 0.2 in the other two regimes. 
\begin{figure}[!htbp]
\centering
\resizebox{\columnwidth}{!}{
{\includegraphics[width=5\textwidth]{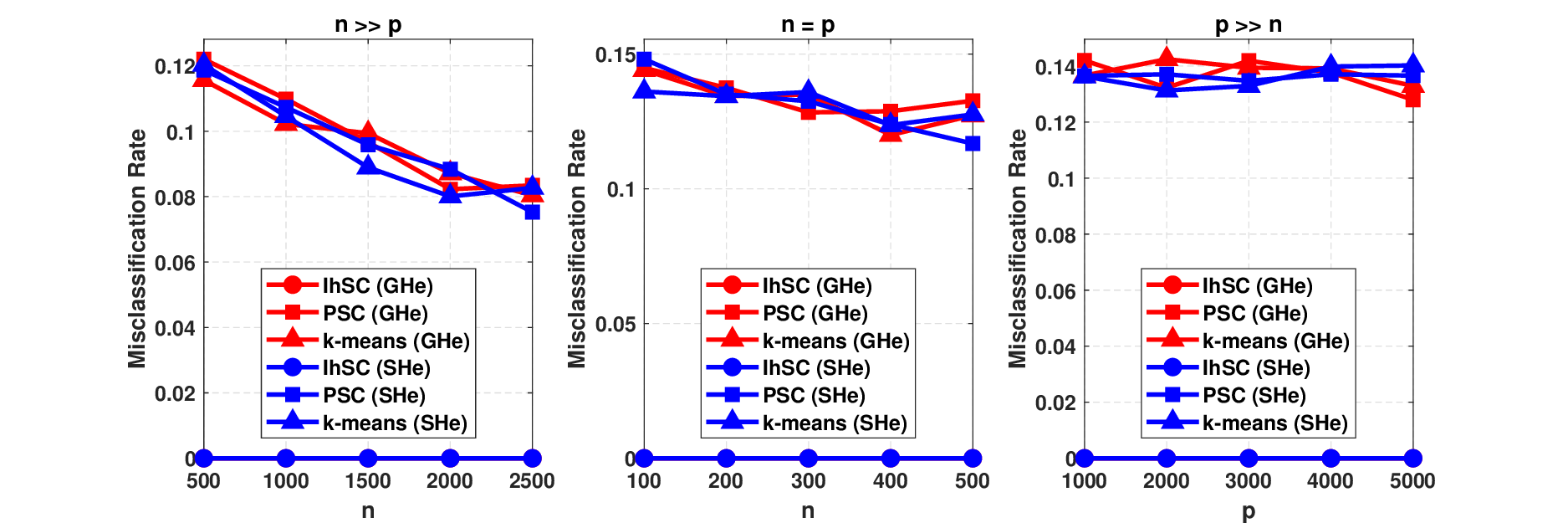}}
}
\resizebox{\columnwidth}{!}{
{\includegraphics[width=5\textwidth]{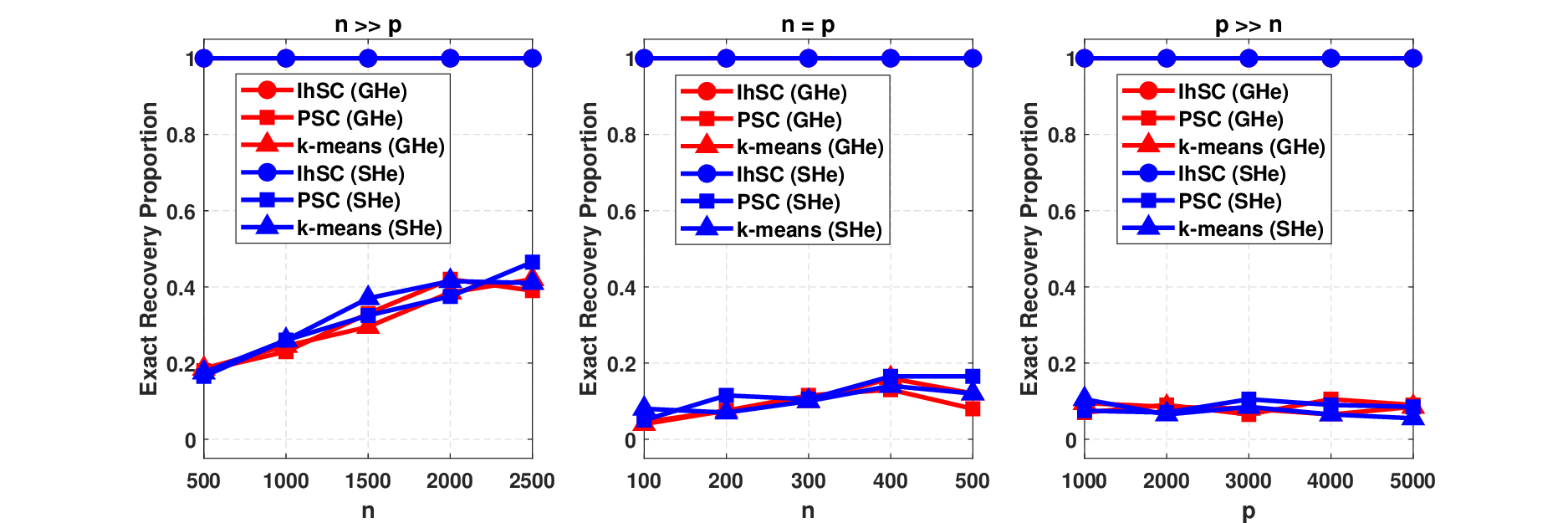}}
}
\caption{Numerical results of Experiment 1.}
\label{fig:ex1} 
\end{figure}
\subsection{Experiment 2: Influence of the Number of Clusters \(K\)}

We now investigate how the algorithm performs as the number of clusters \(K\) increases. Three scenarios are considered:
\begin{itemize}
  \item \textbf{\(n\gg p\) (sample-rich regime):} \(n=5000, p=1000\), \(R=20\), \(\beta=1\).  Set \(K\in\{2,3,\dots,8\}\).
\item \textbf{\(n= p\) (balanced regime):} \(n=1000, p=1000\), \(R=20\), \(\beta=1\).  Set \(K\in\{2,3,\dots,8\}\).
  \item \textbf{\(p\gg n\) (high‑dimensional regime):} \(n=1000, p=5000\), \(R=20\), \(\beta=1\).  Set \(K\in\{2,3,\dots,8\}\).
\end{itemize}

Figure~\ref{fig:ex2} shows the numerical results of this experiment. Our IhSC again achieves perfect exact recovery in every case, with zero misclassification error and a recovery proportion of one. The result is remarkable: even when the clustering problem becomes more complex with more groups, our IhSC continues to separate the clusters without any error. The behavior of PSC and k-means stands in sharp contrast. Both methods suffer from high misclassification errors across all values of \(K\). Their exact recovery proportions are far below one, and increase as the number of clusters increases. 
\begin{figure}[!htbp]
\centering
\resizebox{\columnwidth}{!}{
{\includegraphics[width=5\textwidth]{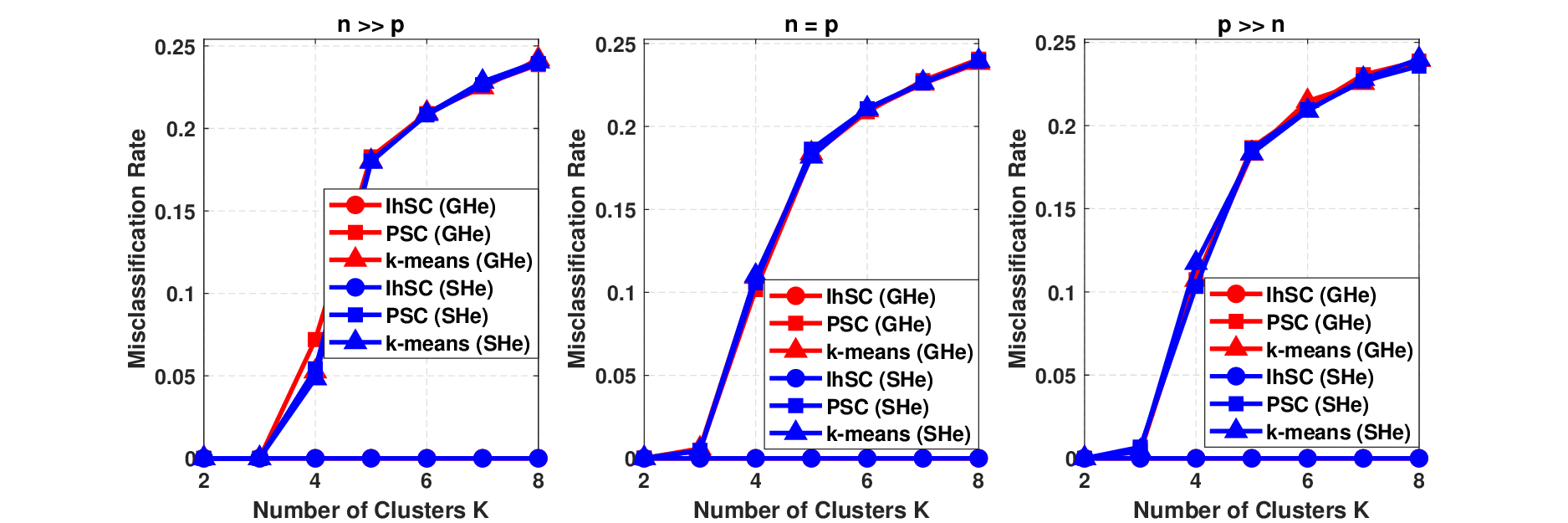}}
}
\resizebox{\columnwidth}{!}{
{\includegraphics[width=5\textwidth]{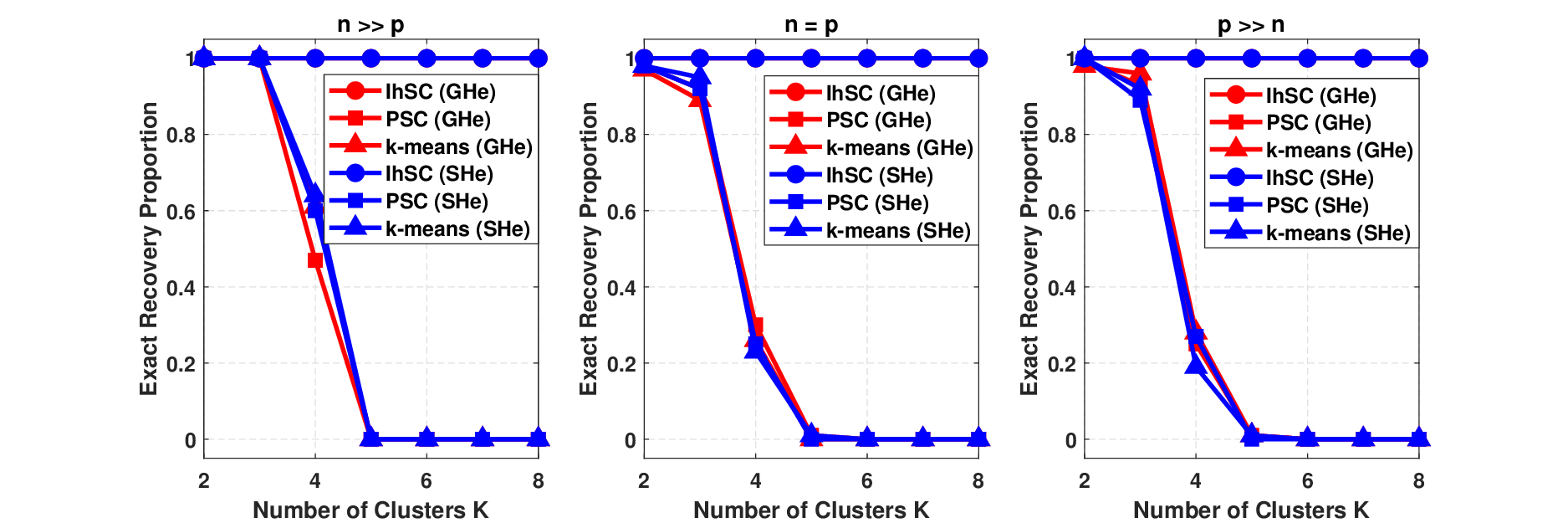}}
}
\caption{Numerical results of Experiment 2.}
\label{fig:ex2} 
\end{figure}
\subsection{Experiment 3: Influence of Heterogeneity Strength \(R\)}
To assess the algorithm’s robustness to strong individual heterogeneity, we fix the balanced setting (\(\beta=1\)) with \(n=500\), \(p=1000\), \(K=3\), and vary \(R\in\{5,10,\ldots,100\}\).  
In this experiment the separation \(\Delta\) is kept constant, equal to the value required for the moderate heterogeneity \(R=5\) under the theoretical condition (i.e., \(\Delta = 3\sqrt{3\log 1000}\,\max\{1,(1000/500)^{1/4}\}\approx16.2408\)). By fixing \(\Delta\) we can observe the effect of increasing \(R\) alone, even though the theoretical sufficient condition may no longer be satisfied for large \(R\). 
\begin{figure}[!htbp]
\centering
\resizebox{\columnwidth}{!}{
{\includegraphics[width=3\textwidth]{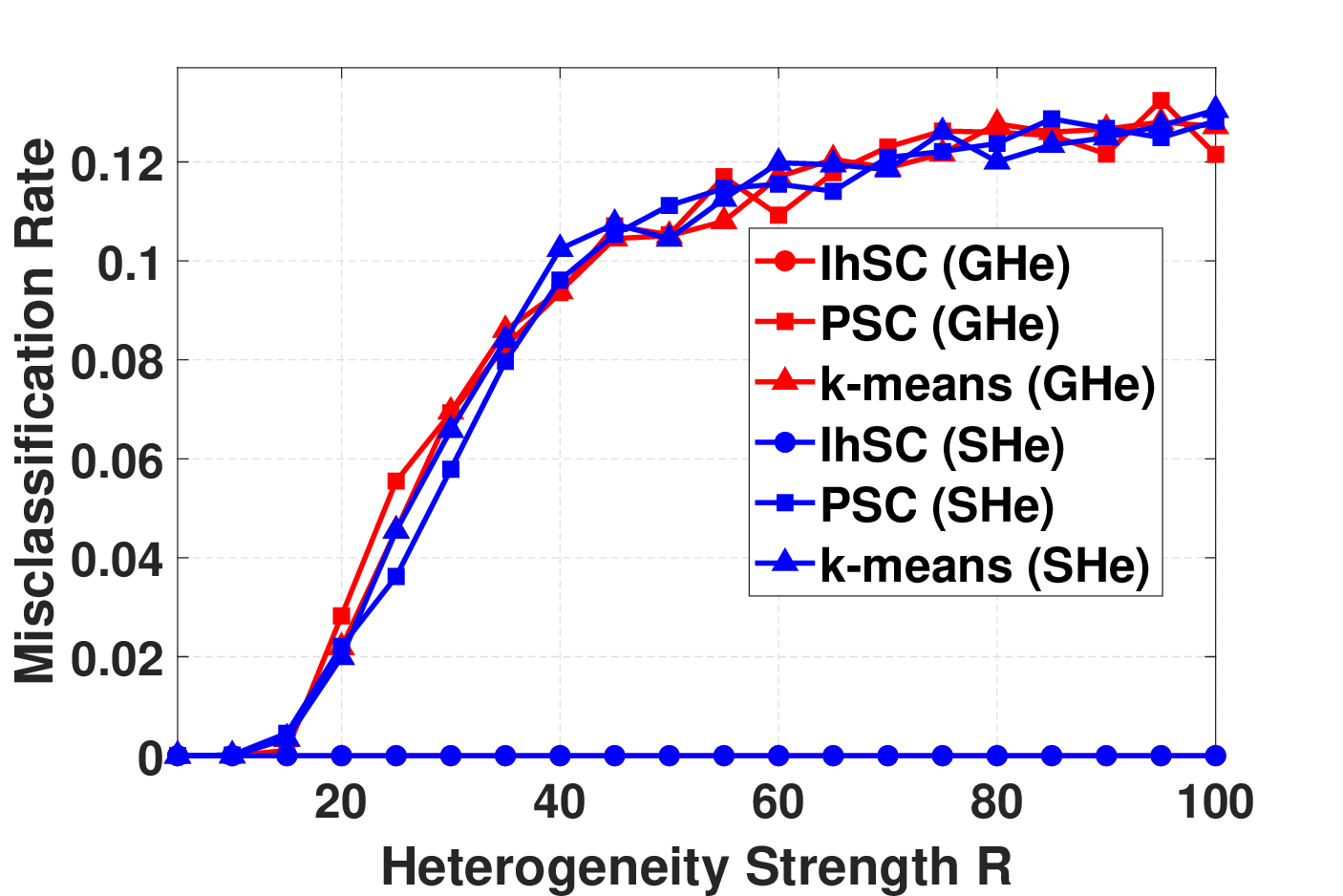}}
{\includegraphics[width=3\textwidth]{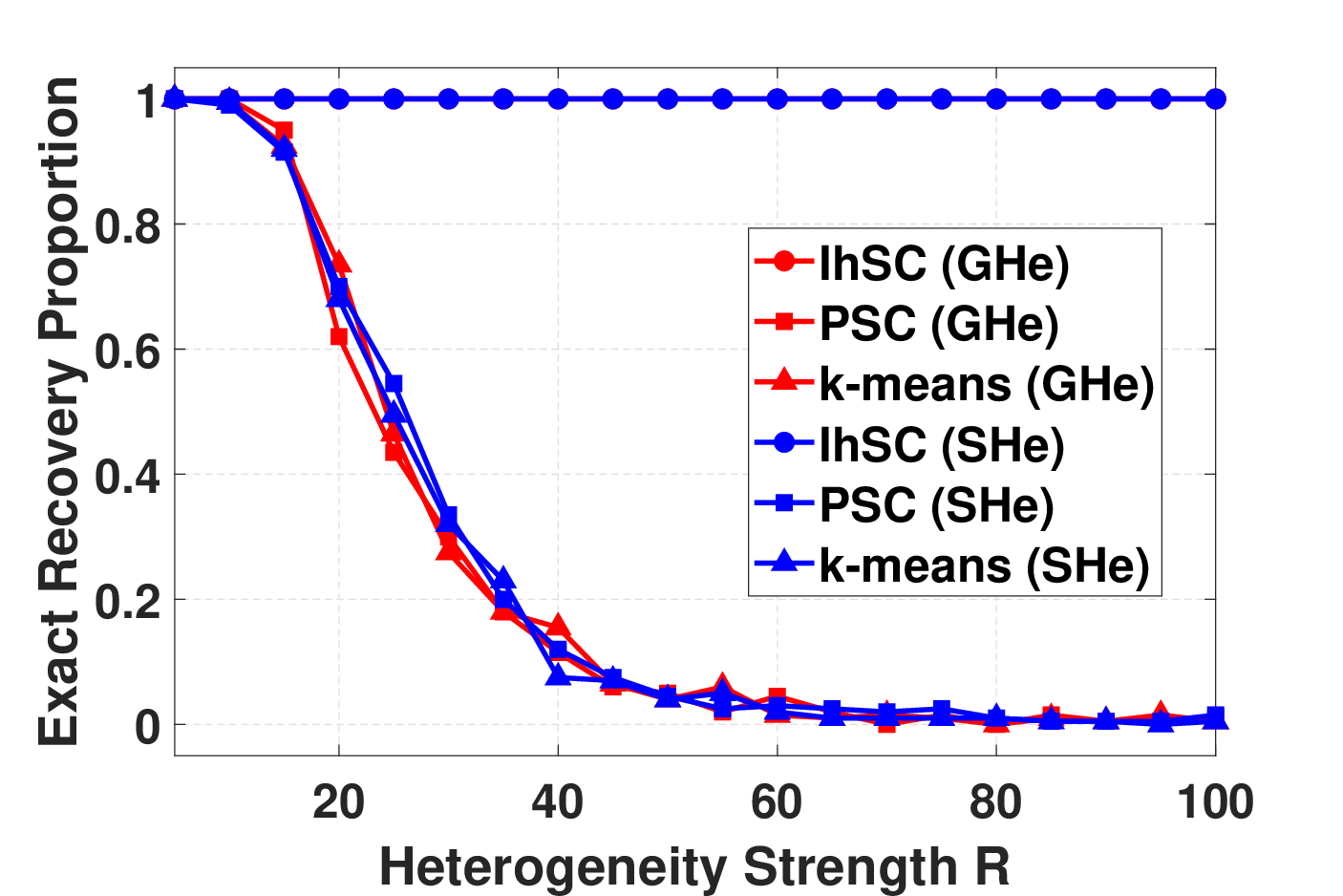}}
}
\caption{Numerical results of Experiment 3.}
\label{fig:ex3} 
\end{figure}

Figure~\ref{fig:ex3} shows how these methods behave when the heterogeneity strength \(R\) grows. Our IhSC achieves perfect recovery for every value of \(R\), with zero misclassification and exact recovery proportion one. This confirms that row normalization completely cancels the effect of the individual scale parameters \(\omega_i\). In contrast, PSC and k-means perform well only when \(R\) is small. As \(R\) increases beyond a moderate level, their misclassification rates rise steadily and their exact recovery proportions drop sharply, eventually approaching zero. 
\subsection{Experiment 4: Influence of Cluster Imbalance \(\beta\)}
Finally, we study the impact of unbalanced cluster sizes.  We fix \(n=200\), \(p=1000\), \(K=3\), \(R=20\), and let \(\beta\in\{0.1,0.2,\ldots,1\}\).  
The separation is fixed to the value required for the balanced case (\(\beta=1\)), i.e., \(\Delta = 3\sqrt{3\log 1000}\,\max\{1,(1000/200)^{1/4}\}\approx20.4217\). This setting allows us to examine how severe imbalance affects the exact recovery capability while keeping the signal strength constant. 
\begin{figure}[!htbp]
\centering
\resizebox{\columnwidth}{!}{
{\includegraphics[width=3\textwidth]{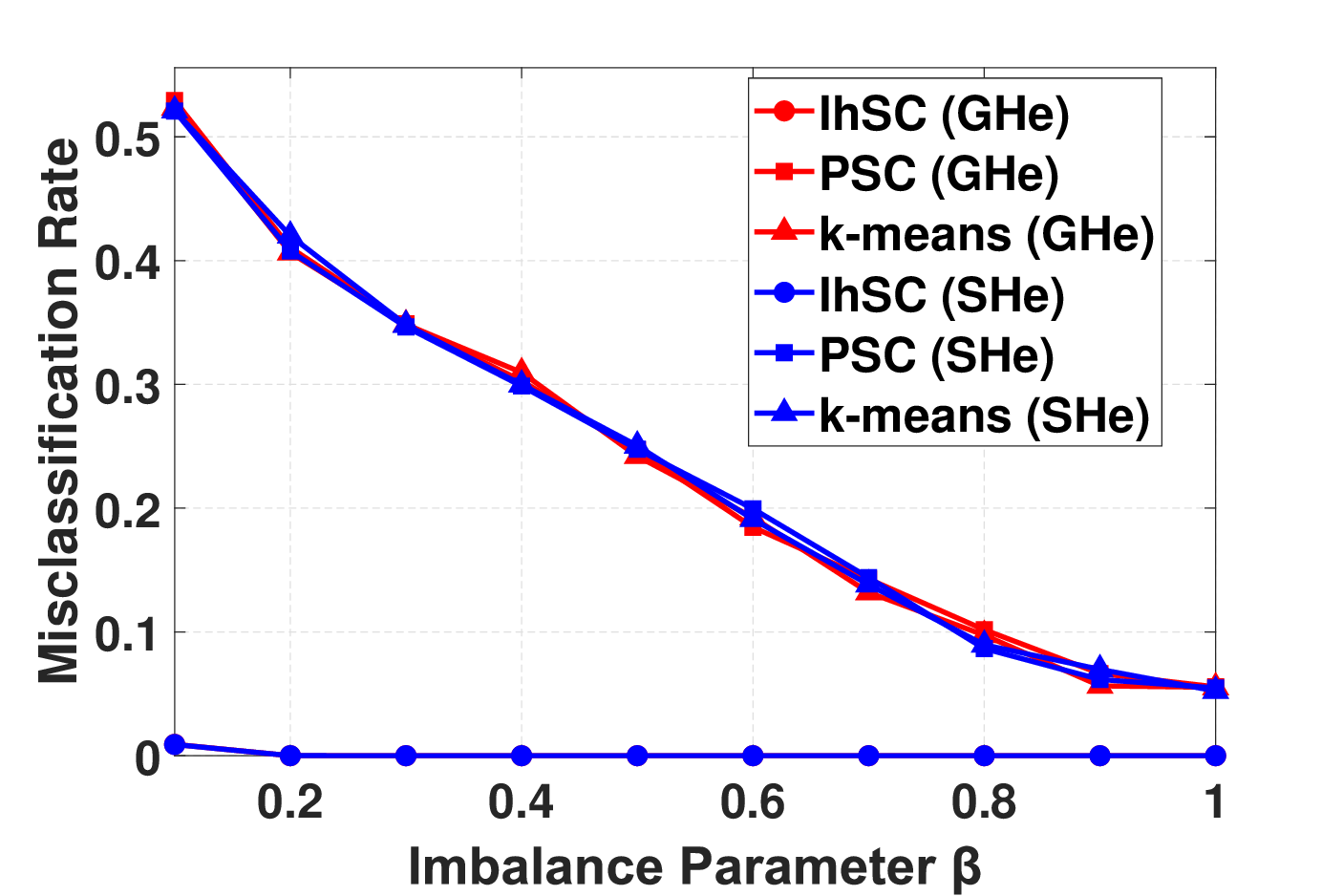}}
{\includegraphics[width=3\textwidth]{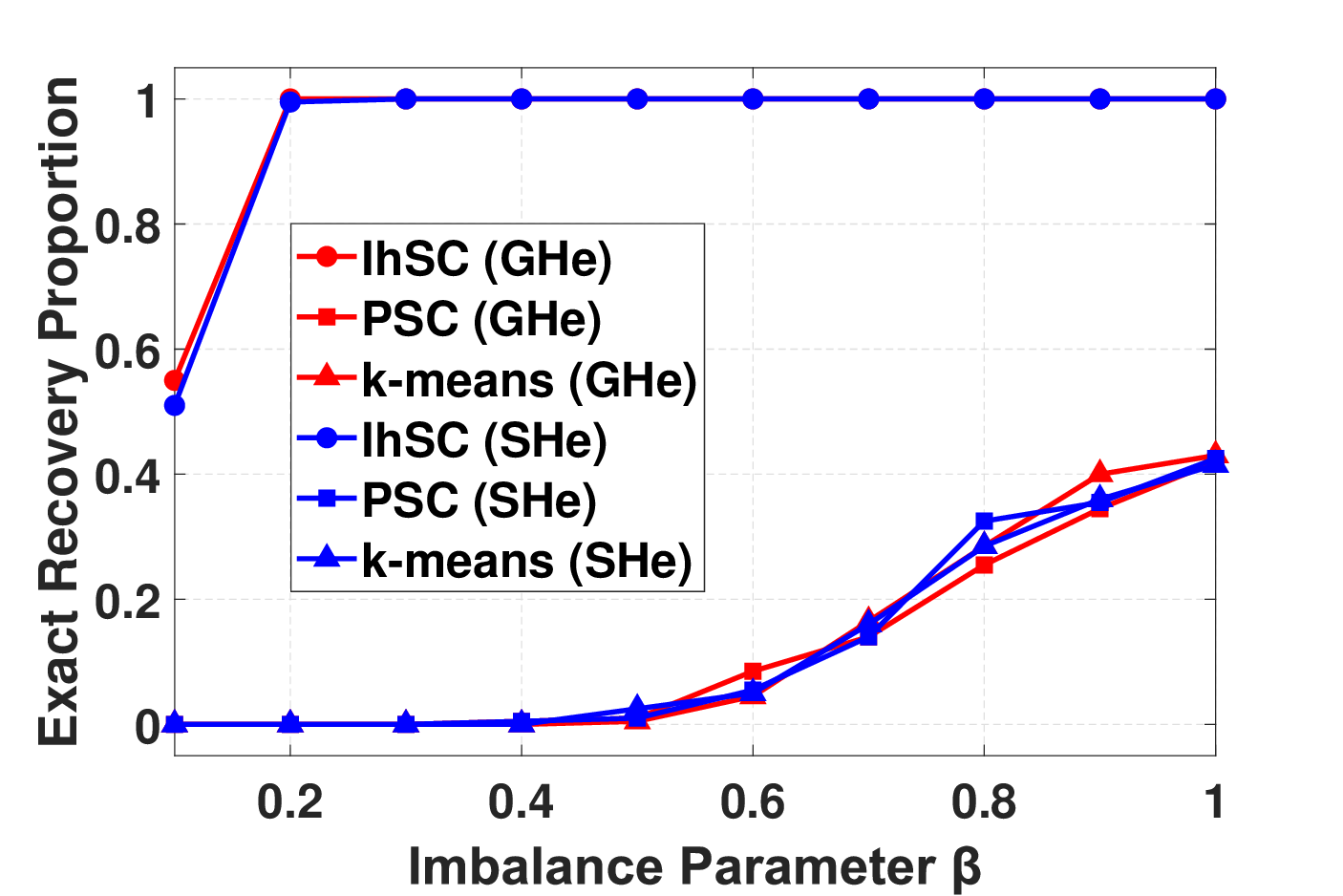}}
}
\caption{Numerical results of Experiment 4.}
\label{fig:ex4} 
\end{figure}

Figure~\ref{fig:ex4} shows how the methods perform as the cluster sizes become more balanced, with \(\beta\) increasing from 0.1 to 1. When the imbalance is extreme (\(\beta = 0.1\)), our IhSC fails to achieve perfect recovery, with its exact recovery proportion falling below 0.5. However, once \(\beta\) reaches 0.2 or larger, IhSC recovers the true labels exactly in every run, maintaining a perfect recovery proportion of one and zero misclassification. This demonstrates that the row normalization step is highly effective as long as the smallest cluster is not extremely tiny. The classical methods, PSC and k-means, improve steadily as the clusters become more balanced: their misclassification rates decline and their exact recovery proportions rise. This improvement is natural because a more even distribution of observations across clusters makes the clustering problem easier. Nevertheless, even at the most balanced setting \(\beta = 1\), neither method achieves a recovery proportion above 0.5. The gap between IhSC and the benchmarks is therefore clear: only IhSC can consistently deliver exact recovery once the smallest cluster has a reasonable size, while the classical methods never reach that level of performance.
\section{Real Data Applications}\label{secrealdata}

This section evaluates the practical performance of our proposed IhSC algorithm on nine real-world datasets. All datasets come with known ground truth class labels, so we can compute exact misclassification rates and compare IhSC with the two competing  methods, k-means and PSC, that were already used in our simulation studies. The results help to demonstrate the effectiveness of our approach. The details of these datasets are provided below \footnote{The Iris, Wine, Seeds, and Dermatology datasets are available from the UCI Machine Learning Repository (\url{https://archive.ics.uci.edu}). The DNA, segment, satimage, usps, and pendigits datasets can be obtained from the LIBSVM Data website (\url{https://www.csie.ntu.edu.tw/~cjlin/libsvmtools/datasets/multiclass.html\#splice}).}:
\begin{itemize}
\item \textbf{Iris.} This classic benchmark gives sepal length, sepal width, petal length, and petal width for three iris species \citep{fisher1936use}. Although the dimension is low (\(p = 4\)), the three classes are not perfectly separable, making it a nontrivial test. There are \(n = 150\) observations and \(K = 3\) true species labels.
\item \textbf{Wine.} This classic classification dataset gives \(p = 13\) chemical attributes of wines from three different cultivars \citep{aeberhard1994comparative}. The attributes include alcohol, malic acid, ash, and others. There are \(n = 178\) samples and \(K = 3\) true cultivar labels.
\item \textbf{Seeds.} This dataset contains geometric measurements of wheat kernels from three varieties \citep{charytanowicz2010complete}. The \(p = 7\) attributes are area, perimeter, compactness, length, width, asymmetry coefficient, and groove length. It has \(n = 210\) samples and \(K = 3\) true variety labels.
\item \textbf{Dermatology.} This dataset contains clinical and histopathological features for differential diagnosis of six erythemato-squamous diseases \citep{guvenir1998learning}. It has \(n = 358\) samples and \(p = 34\) attributes: 12 clinical features, 22 histopathological features. All clinical and histopathological features are rated on a 0 to 3 scale. The true disease labels (\(K = 6\)) are determined by biopsy and clinical evaluation.
\item \textbf{DNA.} This dataset comes from the Statlog DNA splicing benchmark \citep{michie1995machine}. Each sample is a fixed-length DNA sequence encoded by \(p = 180\) binary features that indicate donor and acceptor splice sites. It has \(n = 2000\) samples and \(K = 3\) true biological classes.
\item \textbf{segment.} This Statlog dataset \citep{michie1995machine} consists of image segments derived from outdoor images. It has \(n = 2310\) samples, \(p = 19\) features, and \(K = 7\) classes representing different surface types: brickface, sky, foliage, cement, window, path, and grass.
\item \textbf{satimage.} This multispectral satellite image dataset from the Statlog collection \citep{michie1995machine} contains \(n = 4435\) training pixels, each described by \(p = 36\) spectral bands. The task is to assign each pixel to one of \(K = 6\) land cover classes: red soil, cotton crop, grey soil, damp grey soil, soil with vegetation stubble, and very damp grey soil.
\item \textbf{usps.} The USPS handwritten digit dataset \citep{hull2002database} consists of \(n = 7291\) training grayscale images, each of size \(16 \times 16\) pixels, yielding \(p = 256\) features. The goal is digit classification with \(K = 10\) classes (digits 0-9).
\item \textbf{pendigits.} This UCI Pen-Based Recognition of Handwritten Digits dataset \citep{Alpaydin1998} contains samples from 44 writers. As provided in the LIBSVM repository, it has \(n = 7494\) training samples, \(p = 16\) integer features representing pen-tip $(x,y)$ coordinates during digit writing, and \(K = 10\) target classes (digits 0-9).
\end{itemize}

\begin{table}[htbp]
\centering
\caption{Misclassification rates of IhSC, k-means, and PSC on the nine real datasets.}
\label{tab:real_results}
\begin{tabular}{lccccccccccc}
\toprule
Method &Iris  & Wine & Seeds& Dermatology& DNA&segment&satimage&usps
&pendigits\\
\midrule
IhSC   & \textbf{7/150}& 9/178 & 20/210 &\textbf{ 20/358} & \textbf{353/2000}&\textbf{735/2310}&\textbf{1381/4435}&\textbf{2027/7291}&\textbf{1922/7494}\\
PSC    & 16/150 & 7/178 & \textbf{16/210} & 30/358 & 580/2000&770/2310&1468/4435&2360/7291&2443/7494\\
k-means    & 16/150 &\textbf{ 6/178} & 17/210 & 95/358 & 566/2000&773/2310&1471/4435&2316/7291&2262/7494\\
\bottomrule
\end{tabular}
\end{table}

Table~\ref{tab:real_results} reports the misclassification counts of IhSC, $k$-means, and PSC on these nine real benchmark datasets. IhSC achieves the lowest error on seven of the nine datasets: Iris, Dermatology, DNA, segment, satimage, usps, and pendigits. The improvements are often substantial. On Dermatology, IhSC misclassifies only 20 out of 358 samples, compared to 95 for $k$-means and 30 for PSC, reducing the error by 75 compared to $k$-means and by 10 compared to PSC. On DNA, the advantage is clear: 353 errors versus 566 for $k$-means and 580 for PSC, an improvement of 213 errors over $k$-means and 227 over PSC. On the larger and more challenging multi-class tasks, namely segment (7 classes, 2310 samples), satimage (6 classes, 4435 samples), usps (10 classes, 7291 samples), and pendigits (10 classes, 7494 samples), IhSC consistently outperforms both baselines. The largest absolute gap occurs on pendigits, where IhSC makes 1922 errors versus 2262 for $k$-means and 2443 for PSC, a difference of 340 errors from the better baseline ($k$-means). Even on the small Iris dataset, IhSC cuts the error from 16 to 7, fewer than the 16 errors made by the other methods. On the remaining two datasets, Wine and Seeds, IhSC remains highly competitive. It trails the best performer by only 3 errors on Wine (9 vs. 6) and by 4 on Seeds (20 vs. 16). Overall, IhSC does not merely match existing methods. It consistently outperforms them on the majority of tasks, with particularly large gains on high-dimensional, multi-class, or heterogeneous data. This pattern strongly supports the practical value of explicitly modelling individual heterogeneity in real-world clustering problems.

\section{Conclusion}\label{sec:con}

We have introduced the individual-heterogeneous sub-Gaussian mixture model (Ih-GMM), a new framework that extends the classical Gaussian mixture model by allowing each observation to carry its own scale parameter.  This seemingly minor modification makes the model considerably more realistic for many applications, where data points naturally exhibit different intensities.  

Along with the model, we proposed an efficient spectral algorithm: hollow out the Gram matrix, extract its leading eigenvectors, row-normalize, and then apply \(K\)-means. We proved that, under mild separation conditions on the cluster centers, this algorithm achieves exact recovery of the true labels with high probability. Compared with existing results for classical GMM, our separation condition is significantly weaker, featuring a much milder dependence on the number of clusters \(K\) and on the aspect ratio \(p/n\). Moreover, our analysis allows the number of clusters \(K\) to grow substantially larger than what is permitted in prior works, while still maintaining the exact recovery guarantee. When \(K\) is fixed and the dimension \(p\) is much smaller than the sample size \(n\), our condition reduces to the well-known optimal threshold \(\Delta \gg \sqrt{\log n}\). Numerical experiments demonstrate that our method consistently outperforms its competitors.

The importance of our Ih-GMM goes far beyond a mere technical generalization. The classical GMM has been widely studied for decades, yet its assumption that all observations share the same scale is often unrealistic in practice. By introducing individual scaling parameters \(\omega_i\), Ih-GMM opens a new research direction, much like the degree-corrected stochastic block model (DC-SBM) did for community detection in network analysis. Before DC-SBM, the standard SBM ignored degree heterogeneity, leading to poor fits for real networks. After its introduction, a substantial body of work emerged, covering community detection under degree heterogeneity, parameter estimation, model selection, and computationally scalable algorithms. The same is expected to hold for the proposed Ih-GMM. The model is rich enough to capture real-world heterogeneity, yet sufficiently structured to permit rigorous theoretical analysis. Our simple spectral method provides a solid baseline, and we expect many more advanced algorithms to follow. For practitioners, Ih-GMM offers a trustworthy framework for clustering data with substantial individual variation. The method is easy to implement, computationally efficient, and comes with strong theoretical guarantees for exact recovery, making it a reliable tool for real-world applications.

Like DC-SBM, Ih-GMM is not an end but a beginning. This model opens up many promising directions for future research. Estimating the number of clusters \(K\) when it is unknown is a meaningful yet challenging problem under our Ih‑GMM framework. The presence of individual heterogeneity parameters \(\boldsymbol{\omega}_i\) together with the possibility of heteroskedastic noise makes classical information criteria like AIC or BIC based methods, difficult to apply directly, because the noise distribution interacts with the scaling parameters in a nontrivial way. A natural and theoretically tractable route is to first investigate several important special cases of Ih‑GMM, which serve as stepping stones toward the full model. Specifically, we suggest the following two settings as natural starting points.
\begin{itemize}
  \item \emph{Homoskedastic noise:} 
   \[
   \mathbf{x}_i = \boldsymbol{\omega}_i \boldsymbol{\theta}_{\mathbf{z}_i} + \boldsymbol{\varepsilon}_i,\qquad i=1,2,\dots,n,
   \]  
   where \(\{\boldsymbol{\varepsilon}_i\}\) are independent and identically distributed Gaussian or sub‑Gaussian with mean zero and common covariance \(\sigma^2 \mathbf{I}_p\). This corresponds to the noise homogeneous version of Ih‑GMM, where the only source of individual variation is the scale parameter \(\boldsymbol{\omega}_i\).
  \item \emph{Class specific homoskedastic noise:} 
   \[
   \mathbf{x}_i = \boldsymbol{\omega}_i \boldsymbol{\theta}_{\mathbf{z}_i} + \sigma_{\mathbf{z}_i} \boldsymbol{\varepsilon}_i,\qquad i=1,2,\dots,n,
   \]  
   where \(\{\boldsymbol{\varepsilon}_i\}\) are independent and identically distributed Gaussian or sub‑Gaussian with mean zero and identity covariance, and \(\sigma_{\mathbf{z}_i}>0\) is a class specific noise scale. This allows the noise level to differ across clusters while remaining constant within each cluster, capturing a common form of heteroskedasticity encountered in real applications.
\end{itemize}

Both settings retain the core individual heterogeneity structure of Ih‑GMM while simplifying the noise component. Completing the work on estimating \(K\) under these special cases would provide valuable insights and tools, and is expected to shed light on the more general Ih‑GMM where noise may also vary freely across observations. We therefore view the problem of estimating \(K\) as a promising avenue for future research, with a clear path of progressive generalization.

Another important direction is to design faster algorithms with rigorous theoretical guarantees, for instance by exploiting randomized techniques, as eigen‑decomposition becomes expensive for large‑scale datasets.  Deeper theoretical questions include establishing minimax lower bounds and sharp phase transitions for exact recovery, which would reveal whether the simple spectral method already attains the information-theoretic limit. Model selection between the classical GMM and our Ih-GMM poses a practically relevant yet non-trivial task, requiring hypothesis tests or information criteria that account for the many scale parameters. Robustness to outliers is also critical. Spectral methods that remain provably accurate in the presence of anomalous observations would greatly enhance practical applicability. Beyond these, the development of more sophisticated estimators for the model parameters, such as the scaling parameters and the cluster centers, remains a promising direction for future work. In addition, exploring optimization algorithms such as semidefinite programming relaxation for Ih-GMM could offer new insights into optimal clustering and provide alternative methods with strong theoretical guarantees. Our work demonstrates that even a simple spectral method can handle this realistic model. We hope that our Ih-GMM will become a standard benchmark for clustering under heterogeneity, just as DC-SBM did for network analysis. 
\section*{CRediT authorship contribution statement}
\textbf{Huan Qing} is the sole author of this article.

\section*{Declaration of competing interest}
The author declares no competing interests.

\section*{Data availability}
Data and code will be made available on request.

\appendix
\section{Technical proofs}
\subsection{Proof of Lemma \ref{lem:Vstructure}}
\begin{proof}
Recall that $\mathbf{Z}$ is the membership matrix, $\boldsymbol{\Omega} = \operatorname{diag}(\boldsymbol{\omega}_1,\dots,\boldsymbol{\omega}_n)$, and $\mathbf{D}_{\boldsymbol{\Omega}} = \operatorname{diag}(\tilde n_1,\dots,\tilde n_K)$ with $\tilde n_k = \sum_{i:\mathbf{z}_i=k}\boldsymbol{\omega}_i^2$.  
Define $\tilde{\mathbf{Z}} = \mathbf{Z} \mathbf{D}_{\boldsymbol{\Omega}}^{-1/2}$. We have
\[
\mathbf{Z}^\top = \mathbf{D}_{\boldsymbol{\Omega}}^{1/2} \tilde{\mathbf{Z}}^\top.
\]

The signal matrix is $\mathbf{P} = \boldsymbol{\Theta} \mathbf{Z}^\top \boldsymbol{\Omega}$. Using the above relation gets
\begin{align*}
    \mathbf{P} = \boldsymbol{\Theta} \mathbf{D}_{\boldsymbol{\Omega}}^{1/2} \tilde{\mathbf{Z}}^\top \boldsymbol{\Omega} = (\boldsymbol{\Theta} \mathbf{D}_{\boldsymbol{\Omega}}^{1/2})(\boldsymbol{\Omega} \tilde{\mathbf{Z}})^\top.
\end{align*}

Now perform the singular value decomposition of $\boldsymbol{\Theta} \mathbf{D}_{\boldsymbol{\Omega}}^{1/2}$:
\begin{align}
    \boldsymbol{\Theta} \mathbf{D}_{\boldsymbol{\Omega}}^{1/2} = \mathbf{Q} \boldsymbol{\Sigma} \mathbf{H}^\top, \label{eq:SVD}
\end{align}
where $\mathbf{Q}$ has orthonormal columns, $\boldsymbol{\Sigma} = \operatorname{diag}(\sigma_1,\dots,\sigma_K)$, and $\mathbf{H} \in \mathbb{R}^{K \times K}$ is an orthogonal matrix ($\mathbf{H}^\top \mathbf{H} = \mathbf{I}_K$). Substituting Equation \eqref{eq:SVD} yields
\begin{align}
    \mathbf{P} = \mathbf{Q} \boldsymbol{\Sigma} \mathbf{H}^\top (\boldsymbol{\Omega} \tilde{\mathbf{Z}})^\top = \mathbf{Q} \boldsymbol{\Sigma} (\boldsymbol{\Omega} \tilde{\mathbf{Z}} \mathbf{H})^\top. \label{eq:Pfinal}
\end{align}

We show that $\boldsymbol{\Omega} \tilde{\mathbf{Z}} \mathbf{H}$ has orthonormal columns. Compute
\begin{align}
    (\boldsymbol{\Omega} \tilde{\mathbf{Z}} \mathbf{H})^\top (\boldsymbol{\Omega} \tilde{\mathbf{Z}} \mathbf{H}) = \mathbf{H}^\top \tilde{\mathbf{Z}}^\top \boldsymbol{\Omega}^2 \tilde{\mathbf{Z}} \mathbf{H}. \label{eq:gram}
\end{align}

Using $\tilde{\mathbf{Z}} = \mathbf{Z} \mathbf{D}_{\boldsymbol{\Omega}}^{-1/2}$ gets
\begin{align*}
    \tilde{\mathbf{Z}}^\top \boldsymbol{\Omega}^2 \tilde{\mathbf{Z}} = \mathbf{D}_{\boldsymbol{\Omega}}^{-1/2} \mathbf{Z}^\top \boldsymbol{\Omega}^2 \mathbf{Z} \mathbf{D}_{\boldsymbol{\Omega}}^{-1/2}.
\end{align*}

Because $\mathbf{Z}^\top \boldsymbol{\Omega}^2 \mathbf{Z} = \mathbf{D}_{\boldsymbol{\Omega}}$ is diagonal with
\[
(\mathbf{Z}^\top \boldsymbol{\Omega}^2 \mathbf{Z})_{k\ell} = \sum_{i=1}^n Z_{ik} \boldsymbol{\omega}_i^2 Z_{i\ell} = \delta_{k\ell} \sum_{i:\mathbf{z}_i=k} \boldsymbol{\omega}_i^2 = \delta_{k\ell} \tilde n_k,
\]
where \(\delta_{k\ell}\) denotes the Kronecker delta defined as
\(\delta_{k\ell} = 
\begin{cases}
1, & \text{if } k = \ell,\\
0, & \text{if } k \neq \ell,
\end{cases}\)
we have
\begin{align}
    \tilde{\mathbf{Z}}^\top \boldsymbol{\Omega}^2 \tilde{\mathbf{Z}} = \mathbf{D}_{\boldsymbol{\Omega}}^{-1/2} \mathbf{D}_{\boldsymbol{\Omega}} \mathbf{D}_{\boldsymbol{\Omega}}^{-1/2} = \mathbf{I}_K. \label{eq:inner2}
\end{align}

Inserting Equation \eqref{eq:inner2} into Equation \eqref{eq:gram} gives
\begin{align*}
    (\boldsymbol{\Omega} \tilde{\mathbf{Z}} \mathbf{H})^\top (\boldsymbol{\Omega} \tilde{\mathbf{Z}} \mathbf{H}) = \mathbf{H}^\top \mathbf{I}_K \mathbf{H} = \mathbf{I}_K.
\end{align*}

Thus $\boldsymbol{\Omega} \tilde{\mathbf{Z}} \mathbf{H}$ has orthonormal columns, and from Equation \eqref{eq:Pfinal}, we can identify it as the right singular vector matrix:
\begin{align*}
    \mathbf{U} = \boldsymbol{\Omega} \tilde{\mathbf{Z}} \mathbf{H} = \boldsymbol{\Omega} \mathbf{Z} \mathbf{D}_{\boldsymbol{\Omega}}^{-1/2} \mathbf{H}.
\end{align*}

Now examine the rows of $\mathbf{U}$. For an observation $i$ with $\mathbf{z}_i = k$, $\mathbf{Z}_{i,:} = \mathbf{e}_k^\top$. Consequently, we have
\begin{align*}
    (\mathbf{Z} \mathbf{D}_{\boldsymbol{\Omega}}^{-1/2})_{i,:} = \mathbf{e}_k^\top \mathbf{D}_{\boldsymbol{\Omega}}^{-1/2} = \frac{1}{\sqrt{\tilde n_k}} \mathbf{e}_k^\top.
\end{align*}

Multiplying on the left by $\boldsymbol{\Omega}$ gives
\begin{align*}
    (\boldsymbol{\Omega} \mathbf{Z} \mathbf{D}_{\boldsymbol{\Omega}}^{-1/2})_{i,:} = \boldsymbol{\omega}_i \cdot \frac{1}{\sqrt{\tilde n_j}} \mathbf{e}_j^\top.
\end{align*}

Right‑multiplication by $\mathbf{H}$ then yields
\begin{align}
    \mathbf{U}_{i,:} = \frac{\boldsymbol{\omega}_i}{\sqrt{\tilde n_k}} (\mathbf{e}_k^\top \mathbf{H}) = \frac{\boldsymbol{\omega}_i}{\sqrt{\tilde n_k}} \mathbf{H}_{k,:}, \label{eq:Vrow}
\end{align}
where $\mathbf{H}_{k,:}$ denotes the $k$‑th row of $\mathbf{H}$.

Since $\mathbf{H}$ is orthogonal, each row has unit norm: $\|\mathbf{H}_{k,:}\| = 1$. Hence from Equation \eqref{eq:Vrow}, we get
\begin{align}
    \|\mathbf{U}_{i,:}\| = \frac{\boldsymbol{\omega}_i}{\sqrt{\tilde n_k}}. \label{eq:normV}
\end{align}

Since $\mathbf{U}_*$ by $(\mathbf{U}_*)_{i,:} = \mathbf{U}_{i,:} / \|\mathbf{U}_{i,:}\|$, using Equations \eqref{eq:Vrow} and \eqref{eq:normV} obtains
\begin{align*}
    (\mathbf{U}_*)_{i,:} = \frac{\frac{\boldsymbol{\omega}_i}{\sqrt{\tilde n_k}} \mathbf{H}_{k,:}}{\frac{\boldsymbol{\omega}_i}{\sqrt{\tilde n_k}}} = \mathbf{H}_{k,:}.
\end{align*}

Thus $(\mathbf{U}_*)_{i,:} = \mathbf{H}_{\mathbf{z}_i,:}$, which in matrix form is $\mathbf{U}_* = \mathbf{Z} \mathbf{H}$. Because each row of $\mathbf{Z}$ contains exactly one $1$, the $i$-th row of $\mathbf{U}_*$ equals the row of $\mathbf{H}$ indexed by $\mathbf{z}_i$. Consequently, all rows belonging to the same cluster are identical.

Finally, for two distinct clusters $k \neq \ell$ and any $i$ with $\mathbf{z}_i = k$, $i'$ with $\mathbf{z}_{i'} = \ell$,
\begin{align*}
    \|(\mathbf{U}_*)_{i,:} - (\mathbf{U}_*)_{i',:}\|^2 = \|\mathbf{H}_{k,:}\|^2 + \|\mathbf{H}_{\ell,:}\|^2 - 2\langle \mathbf{H}_{k,:}, \mathbf{H}_{\ell,:}\rangle = 1+1-0 = 2,
\end{align*}
so the Euclidean distance between rows of different clusters is $\sqrt{2}$. This completes the proof.
\end{proof}

\subsection{Proof of Theorem \ref{thm:exact}}
\begin{proof}
We apply Theorem 1 of \citep{cai2021subspace} to the matrix $\mathbf{X}^\top = \mathbf{P}^\top + \mathbf{E}^\top$, where $\mathbf{P}^\top = \boldsymbol{\Omega}\mathbf{Z}\boldsymbol{\Theta}^\top$ and $\mathbf{N}= \mathbf{E}^\top$ is the noise matrix with independent sub‑Gaussian entries satisfying $\|\mathbf{N}_{i,j}\|_{\psi_2}\le \eta$ for all $i,j$.

First, we verify conditions needed by Theorem 1 of \citep{cai2021subspace}. Set $d_1=n$, $d_2=p$, and the sampling rate $p_{\text{samp}}=1$ (full observation). For any \(c>0\), we set the threshold \(R = \eta\sqrt{2c\log d}\).  
Since each entry \(\mathbf{N}_{i,j}\) is sub‑Gaussian with \(\|\mathbf{N}_{i,j}\|_{\psi_2}\le\eta\), the standard tail bound gives  
\[
\mathbb{P}(|\mathbf{N}_{i,j}| \ge t) \le 2\exp\Bigl(-\frac{t^2}{2\eta^2}\Bigr),\qquad \forall t\ge0.
\]  

Substituting \(t = R\) yields  
\[
\mathbb{P}(|\mathbf{N}_{i,j}| > R) \le 2\exp\Bigl(-\frac{(\eta\sqrt{2c\log d})^2}{2\eta^2}\Bigr)
= 2\exp\Bigl(-\frac{2c\eta^2\log d}{2\eta^2}\Bigr)
= 2\exp\bigl(-c\log d\bigr)=2 d^{-c}.
\] 

Taking $c=12$ ensures the condition (10) of \citep{cai2021subspace} holds with $R$ satisfying $\frac{R^2}{\eta^2}\preceq\min\{\sqrt{np},p\}/\log d$ by Condition (\ref{MainNPKCond}). Hence Assumption 2 of \citep{cai2021subspace} is satisfied. Thus, the conditions in Equation (15) of \citep{cai2021subspace} with $p_{\text{samp}}=1$ become
\begin{align}\label{cai15}
1 &\ge c_0\max\left\{\frac{\mu\kappa_{\mathbf{P}}^4 K\log^2 d}{\sqrt{np}},\frac{\mu\kappa_{\mathbf{P}}^8 K\log^2 d}{p}\right\},\frac{\eta}{\sigma_K(\mathbf{P})}\le c_1\min\left\{\frac{1}{\kappa_{\mathbf{P}}\sqrt[4]{np}\sqrt{\log d}},\frac{1}{\kappa_{\mathbf{P}}^3}\sqrt{\frac{1}{n\log d}}\right\},
K\le c_2\frac{n}{\mu_1\kappa_{\mathbf{P}}^4}.
\end{align}

Our aim is to make the three inequalities in Equation (\ref{cai15}) always hold by using proper  sample size conditions and separation conditions. For $\sigma_{K}(\mathbf{P}), \kappa_{\mathbf{P}}$, and $\mu_1$, we have:
\begin{itemize}
  \item \(\sigma_K(\mathbf{P}) \ge \frac{\boldsymbol{\omega}_{\min}\Delta}{2\kappa}\sqrt{\frac{\beta n}{K}}\) by Lemma~\ref{lem:sigmaK}.
  \item \(\kappa_{\mathbf{P}} \le \kappa \frac{\omega_{\max}}{\omega_{\min}} \sqrt{\tau}\) by Lemma \ref{lem:kapparel}.
  \item By Equation (\ref{eq:normV}) and Lemma \ref{lem:Vstructure}, we have \(\|{\mathbf{U}}^\top\mathbf{e}_i\|_2^2 = \|\mathbf{U}_{i,:}\|_2^2 = \frac{\boldsymbol{\omega}_i^2}{\tilde{n}_k}\). For every $i$, we have
\(\|{\mathbf{U}}^\top\mathbf{e}_i\|_2^2 \le \frac{\boldsymbol{\omega}_{\max}^2}{\boldsymbol{\omega}_{\min}^2\beta n/K} = \frac{\boldsymbol{\omega}_{\max}^2 K}{\boldsymbol{\omega}_{\min}^2\beta n}\),
which gives \(\mu_1 = \frac{n}{K}\max_{i\in[n]}\|{\mathbf{U}}^\top\mathbf{e}_i\|_2^2 \le \frac{n}{K}\cdot\frac{\boldsymbol{\omega}_{\max}^2 K}{\boldsymbol{\omega}_{\min}^2\beta n} = \frac{\boldsymbol{\omega}_{\max}^2}{\boldsymbol{\omega}_{\min}^2\beta}\).
\end{itemize}

Therefore, to make the three inequalities in Equation (\ref{cai15}) always hold, we need
\begin{align*}
  np&\gg\frac{\mu^2\tau^4\kappa^8\boldsymbol{\omega}^8_{\max}}{\boldsymbol{\omega}^8_{\min}}K^2\log^4 d,\\
  p&\gg\frac{\mu\tau^4\kappa^8\boldsymbol{\omega}^8_{\max}}{\boldsymbol{\omega}^8_{\min}}K\log^2 d,\\
  n&\gg\frac{\tau^2\kappa^4\boldsymbol{\omega}^6_{\max}}{\boldsymbol{\omega}^6_{\min}}K,\\
  \Delta&\gg \frac{\eta\kappa^2\boldsymbol{\omega}_{\max}\sqrt{\tau}}{\boldsymbol{\omega}^2_{\min}\sqrt{\beta}}(\frac{p}{n})^{\frac{1}{4}}\sqrt{K\log d},\\
  \Delta&\gg\frac{\eta\kappa^4\boldsymbol{\omega}^3_{\max}\sqrt{\tau^3}}{\boldsymbol{\omega}^4_{\min}\sqrt{\beta}}\sqrt{K\log d},\\
\end{align*}
where these inequalities always hold by the sample size conditions and separation conditions stated in the theorem. Thus, all conditions of Theorem 1 in \citep{cai2021subspace} are satisfied.

Next we apply Theorem 1 of \citep{cai2021subspace} to obtain the $\ell_{2,\infty}$ distance between $\hat{\mathbf{U}}$ and $\mathbf{U}$.
Under the verified conditions, \citep[Theorem~1]{cai2021subspace} states that with probability at least $1-O(d^{-10})$, there exists an orthogonal matrix $\mathscr{O}\in\mathbb{R}^{K\times K}$ such that
\begin{align*}
\|\hat{\mathbf{U}}\mathscr{O} - \mathbf{U}\|_{2,\infty} \preceq \kappa_{\mathbf{P}}^2\sqrt{\frac{\mu K}{n}}\;\mathcal{E}_{\text{gen}},
\end{align*}
with
\[
\mathcal{E}_{\text{gen}} = \frac{\mu_1\kappa_{\mathbf{P}}^2 K}{n} + \frac{\eta^2\sqrt{np}\log d}{\sigma_K^2(\mathbf{P})} + \frac{\eta\kappa_{\mathbf{P}}\sqrt{n\log d}}{\sigma_K(\mathbf{P})},
\]
where the  first two
terms on the right-hand side of Equation (17) in Theorem 1 of \citep{cai2021subspace} are removed because there is no missing data.

From Lemma~\ref{lem:Vstructure}, $\min_i\|\mathbf{U}_{i,:}\| \ge \frac{\boldsymbol{\omega}_{\min}}{\boldsymbol{\omega}_{\max}}\sqrt{\frac{1}{n_{\max}}}$. For any vectors $a,b$ with $b\neq0$, $\|a/\|a\|-b/\|b\|\| \le 2\|a-b\|/\|b\|$. Hence, we have
\begin{align*}
\|\hat{\mathbf{U}}_*\mathscr{O} - \mathbf{U}_*\|_{2,\infty} &\le \frac{2}{\min_i\|\mathbf{U}_{i,:}\|}\,\|\hat{\mathbf{U}}\mathscr{O} - \mathbf{U}\|_{2,\infty}\le 2\frac{\boldsymbol{\omega}_{\max}}{\boldsymbol{\omega}_{\min}}\sqrt{n_{\max}}\; \|\hat{\mathbf{U}}\mathscr{O} - \mathbf{U}\|_{2,\infty}.
\end{align*}

Substituting the bound for $\|\hat{\mathbf{U}}\mathscr{O} - \mathbf{U}\|_{2,\infty}$ gets
\begin{align*}
\|\hat{\mathbf{U}}_*\mathscr{O} - \mathbf{U}_*\|_{2,\infty}
&\preceq 2\frac{\boldsymbol{\omega}_{\max}}{\boldsymbol{\omega}_{\min}}\sqrt{n_{\max}}\cdot \kappa_{\mathbf{P}}^2\sqrt{\frac{\mu K}{n}}\;\mathcal{E}_{\text{gen}}= 2 \kappa_{\mathbf{P}}^2\frac{\boldsymbol{\omega}_{\max}}{\boldsymbol{\omega}_{\min}}\sqrt{\frac{Kn_{\max}\mu }{n}}\;\mathcal{E}_{\text{gen}}.
\end{align*}

Since \(\kappa_{\mathbf{P}} \le \kappa \frac{\omega_{\max}}{\omega_{\min}} \sqrt{\tau}\) by Lemma \ref{lem:kapparel}, we have
\begin{align*}
\|\hat{\mathbf{U}}_*\mathscr{O} - \mathbf{U}_*\|_{2,\infty} \preceq \frac{\tau\kappa^2\boldsymbol{\omega}_{\max}^3}{\boldsymbol{\omega}_{\min}^3}\sqrt{\frac{\mu Kn_{\max}}{n}}\cdot \mathcal{E}_{\text{gen}}.
\end{align*}

By \(\sigma_K(\mathbf{P}) \ge \frac{\boldsymbol{\omega}_{\min}\Delta}{2\kappa}\sqrt{\frac{\beta n}{K}}\), \(\kappa_{\mathbf{P}} \le \kappa \frac{\omega_{\max}}{\omega_{\min}} \sqrt{\tau}\), \(\mu_1 \leq \frac{\boldsymbol{\omega}_{\max}^2}{\boldsymbol{\omega}_{\min}^2\beta}\), and the form of \(\mathcal{E}_{\text{gen}}\), we have
\begin{align*}
\|\hat{\mathbf{U}}_*\mathscr{O} - \mathbf{U}_*\|_{2,\infty} &\preceq \frac{\tau\kappa^2\boldsymbol{\omega}_{\max}^3}{\boldsymbol{\omega}_{\min}^3}\sqrt{\frac{\mu Kn_{\max}}{n}}(\frac{\tau\kappa^2K\boldsymbol{\omega}^4_{\max}}{\boldsymbol{\omega}^4_{\min}\beta n}+\frac{4\kappa^2\eta^2 K\sqrt{p}\log d}{\boldsymbol{\omega}^2_{\min}\beta\sqrt{n}\Delta^2}+\frac{2\kappa^2\boldsymbol{\omega}_{\max}\eta\sqrt{\tau K\log d}}{\boldsymbol{\omega}^2_{\min}\sqrt{\beta}\Delta})\\
&=\frac{\tau^2\kappa^4\boldsymbol{\omega}^7_{\max}\sqrt{\mu}}{\boldsymbol{\omega}^7_{\min}\beta}\sqrt{\frac{K^3n_{\max}}{n^3}}+\frac{4\tau\kappa^4\eta^2\boldsymbol{\omega}^3_{\max}\sqrt{\mu}}{\boldsymbol{\omega}^5_{\min}\beta}\frac{\log d\sqrt{K^3pn_{\max}}}{\Delta^2n}+\frac{2\kappa^4\boldsymbol{\omega}^4_{\max}\eta\sqrt{\tau^3\mu}}{\boldsymbol{\omega}^5_{\min}\sqrt{\beta}}\frac{K\sqrt{n_{\max}\log d}}{\Delta\sqrt{n}}\equiv\tilde{\mathcal{E}}_{\text{gen}}.
\end{align*}

By the sample size conditions (\ref{MainNPKCond}) and separation conditions (\ref{MainDeltaCond}), we have $\tilde{\mathcal{E}}_{\text{gen}}=o(1)$, which guarantees
\[
\|\hat{\mathbf{U}}_*\mathscr{O}- \mathbf{U}_*\|_{2,\infty} < \frac{\sqrt{2}}{2}.
\]

By Lemma~\ref{lem:Vstructure}, the rows of $\mathbf{U}_*$ satisfy:
\begin{itemize}
  \item  if $\mathbf{z}_i = \mathbf{z}_j = k$, then $\mathbf{U}_{*,i,:} = \mathbf{U}_{*,j,:}$;
  \item if $\mathbf{z}_i = k$, $\mathbf{z}_j = \ell$ and $k \neq \ell$, then $\|\mathbf{U}_{*,i,:} - \mathbf{U}_{*,j,:}\| = \sqrt{2}$.
\end{itemize}

Thus each cluster $k$ has a unique center $\mathbf{c}_k := \mathbf{U}_{*,i,:}$ for any $i$ with $\mathbf{z}_i = k$. We have shown that with probability at least $1-O(d^{-10})$,
\[
\|\hat{\mathbf{U}}_*\mathscr{O} - \mathbf{U}_*\|_{2,\infty} < \frac{\sqrt{2}}{2}.
\]

Hence for every $i\in[n]$, we have
\[
\|(\hat{\mathbf{U}}_*\mathscr{O})_{i,:} - \mathbf{U}_{*,i,:}\| < \frac{\sqrt{2}}{2}.
\]

Fix an index $i$ with true label $k$, and let $\ell \neq k$.  Pick any $j$ with $\mathbf{z}_j = \ell$ (such $j$ exists because all clusters are non‑empty because $\beta>0$).  By the triangle inequality, we have
\[
\|(\hat{\mathbf{U}}_*\mathscr{O})_{i,:} - \mathbf{c}_\ell\|
\ge \|\mathbf{c}_k - \mathbf{c}_\ell\| - \|(\hat{\mathbf{U}}_*\mathscr{O})_{i,:} - \mathbf{c}_k\|
> \sqrt{2} - \frac{\sqrt{2}}{2} = \frac{\sqrt{2}}{2}.
\]

Thus each row $(\hat{\mathbf{U}}_*\mathscr{O})_{i,:}$ is strictly closer to its own center $\mathbf{c}_k$ than to any other center $\mathbf{c}_\ell$ ($\ell\neq k$).  Therefore, applying $K$-means to the rows of $\hat{\mathbf{U}}_*\mathscr{O}$ recovers the true cluster labels exactly.  Since $\mathscr{O}$ is an orthogonal matrix, multiplying all rows by $\mathscr{O}$ preserves distances and only permutes the cluster labels. Consequently, $K$-means applied directly to the rows of $\hat{\mathbf{U}}_*$ yields the same clustering up to the permutation induced by $\mathscr{O}$.  Hence $\ell(\hat{\mathbf{z}},\mathbf{z}) = 0$.
\end{proof}
\subsection{Proof of Corollary \ref{cor:balanced}}
\begin{proof}
When $\tau = O(1)$, all cluster sizes are balanced. Then $\beta = O(1)$ and $n_{\max} = O(n / K)$. Hence, $K n_{\max} / n = O(1)$. Substituting $\tau = O(1)$, $\beta = O(1)$, $K n_{\max} / n = O(1)$, $\eta = O(1)$, $\mu = O(1)$, and $\kappa = O(1)$ into Equations \eqref{MainNPKCond} and \eqref{MainDeltaCond} yields the results in this corollary.
\end{proof}
\subsection{Proof of Lemma \ref{lem:incoherence}}
\begin{proof}
We establish the two bounds separately.

\paragraph{Upper bound for $\mu_0$}
From $\mathbf{P}= \boldsymbol{\Theta}\mathbf{Z}^\top\boldsymbol{\Omega}$ we have $\mathbf{P}_{j,i} = \boldsymbol{\omega}_i\,\boldsymbol{\Theta}_{j,\mathbf{z}_i}$ for $i\in[n],\,j\in[p]$. Using Assumption \ref{ass:entrynorm},
\begin{align*}
\max_{i,j}|\mathbf{P}_{j,i}| \le \boldsymbol{\omega}_{\max}\cdot\max_{j,k}|\boldsymbol{\Theta}_{j,k}| \le \boldsymbol{\omega}_{\max}C_{\boldsymbol{\Theta}},
\end{align*}
which gives
\begin{align*}
pn\max_{i,j}|\mathbf{P}_{j,i}|^2 \le pn\,\boldsymbol{\omega}_{\max}^2 C_{\boldsymbol{\Theta}}^2.
\end{align*}

For the denominator, we have
\begin{align*}
\|\mathbf{P}\|_{\mathrm{F}}^2 &= \sum_{i=1}^n\sum_{j=1}^p \boldsymbol{\omega}_i^2\boldsymbol{\Theta}_{j,\mathbf{z}_i}^2 = \sum_{k=1}^K\Bigl(\sum_{i:\mathbf{z}_i=k}\boldsymbol{\omega}_i^2\Bigr)\|\boldsymbol{\theta}_k\|^2 = \sum_{k=1}^K \tilde{n}_k\|\boldsymbol{\theta}_k\|^2.
\end{align*}

By basic algebra, we have $\tilde{n}_k \ge \boldsymbol{\omega}_{\min}^2 n_k \ge \boldsymbol{\omega}_{\min}^2\beta n/K$. By Assumption~\ref{ass:entrynorm}, $\|\boldsymbol{\theta}_k\|^2 \ge c^2 p$. Therefore, we have
\begin{align*}
\|\mathbf{P}\|_{\mathrm{F}}^2 \ge \sum_{k=1}^K \bigl(\boldsymbol{\omega}_{\min}^2\beta n/K\bigr)\,c^2 p = \boldsymbol{\omega}_{\min}^2\beta c^2 n p.
\end{align*}

Combining the estimates gives
\begin{align*}
\mu_0 = \frac{pn\max_{j,i}|\mathbf{P}_{j,i}|^2}{\|\mathbf{P}\|_{\mathrm{F}}^2}
\le \frac{pn\,\omega_{\max}^2 C_{\boldsymbol{\Theta}}^2}{\omega_{\min}^2\beta c^2 n p}
= \frac{\boldsymbol{\omega}_{\max}^2 C_{\boldsymbol{\Theta}}^2}{\boldsymbol{\omega}_{\min}^2\beta c^2}.
\end{align*}

\paragraph{Upper bound for $\mu_2$}
Note that $\boldsymbol{\Omega}\mathbf{Z}\in\mathbb{R}^{n\times K}$ has full column rank because $\mathbf{Z}$ is a membership matrix and $\omega_i>0$; indeed, its columns are orthogonal and nonzero. Consequently, we have
\begin{align*}
\operatorname{col}(\mathbf{P}) = \operatorname{col}(\boldsymbol{\Theta}) = \mathcal{S},
\end{align*}
where $\mathcal{S}=\operatorname{col}(\boldsymbol{\Theta})$ is the column space of $\boldsymbol{\Theta}$. The columns of $\mathbf{V}$ form an orthonormal basis of $\mathcal{S}$. Hence for any $j\in[p]$, the projection of $\mathbf{e}_j$ onto $\mathcal{S}$ satisfies
\begin{align*}
\mathcal{P}_{\mathcal{S}}\mathbf{e}_j = \mathbf{V}{\mathbf{V}}^\top\mathbf{e}_j,
\end{align*}
and therefore
\begin{align*}
\|\mathcal{P}_{\mathcal{S}}\mathbf{e}_j\|_2^2 = \mathbf{e}_j^\top\mathbf{V}{\mathbf{V}}^\top\mathbf{e}_j = \|{\mathbf{V}}^\top\mathbf{e}_j\|_2^2.
\end{align*}

By Assumption~\ref{ass:subspace}, we have 
\begin{align*}
\mu_2 = \frac{p}{K}\max_{j}\|{\mathbf{V}}^\top\mathbf{e}_j\|_2^2 \le \frac{p}{K}\cdot\frac{C_{\mathrm{inc}}K}{p} = C_{\mathrm{inc}}.
\end{align*}
\end{proof}
\section{Technical lemmas}
\begin{lem}\label{lem:sigmaK}
Under the Ih-GMM model, we have
\begin{align*}
\sigma_K(\mathbf{P}) \ge \frac{\boldsymbol{\omega}_{\min}\Delta}{2\kappa}\sqrt{\frac{\beta n}{K}}.
\end{align*}
\end{lem}
\begin{proof}[Proof of Lemma \ref{lem:sigmaK}]
Let $\tilde n_k = \sum_{i:\mathbf{z}_i=k}\boldsymbol{\omega}_i^2$ and $\tilde n_{(K)} = \min_k\tilde n_k$.  
Since $\mathbf{P}= \boldsymbol{\Theta}\mathbf{Z}^\top\boldsymbol{\Omega}$, we have $\sigma_K(\mathbf{P}) \ge \sigma_K(\boldsymbol{\Theta})\,\sigma_K(\mathbf{Z}^\top\boldsymbol{\Omega})$.  
The singular values of $\mathbf{Z}^\top\boldsymbol{\Omega}$ are $\sqrt{\tilde n_1},\dots,\sqrt{\tilde n_K}$, so $\sigma_K(\mathbf{Z}^\top\boldsymbol{\Omega}) = \sqrt{\tilde n_{(K)}}$.  
Recall that $n_k\ge\beta n/K$ and ${\omega}_i\ge\boldsymbol{\omega}_{\min}$, hence $\tilde n_k\ge \boldsymbol{\omega}_{\min}^2 \beta n/K$. Thus $\sigma_K(\mathbf{Z}^\top\boldsymbol{\Omega}) \ge \boldsymbol{\omega}_{\min}\sqrt{\beta n/K}$.  
For $\boldsymbol{\Theta}$, the separation $\Delta$ gives $\|\boldsymbol{\theta}_j-\boldsymbol{\theta}_\ell\|\ge\Delta$. The triangle inequality implies $\Delta\le 2\|\boldsymbol{\Theta}\|$, so $\sigma_1(\boldsymbol{\Theta})\ge\Delta/2$. Hence $\sigma_K(\boldsymbol{\Theta}) = \sigma_1(\boldsymbol{\Theta})/\kappa \ge \Delta/(2\kappa)$. Combining the bounds yields the result. 
\end{proof}

\begin{lem}\label{lem:kapparel}
Let \(\kappa_{\mathbf{P}} = \sigma_1(\mathbf{P}) / \sigma_K(\mathbf{P})\) be the condition number of \(\mathbf{P} \) and . Then
\[
\kappa_{\mathbf{P}}
\;\le\;
\kappa \,\frac{\boldsymbol{\omega}_{\max}\sqrt{\tau}}{\boldsymbol{\omega}_{\min}} .
\]
\end{lem}
\begin{proof}[Proof of Lemma \ref{lem:kapparel}]
Define $\mathbf{B} = \boldsymbol{\Omega}\mathbf{Z} \in \mathbb{R}^{n\times K}$. Because $\mathbf{Z}$ is a membership matrix and $\boldsymbol{\Omega}=\operatorname{diag}(\omega_1,\dots,\omega_n)$,
\[
\mathbf{B}^\top \mathbf{B} = \mathbf{Z}^\top \boldsymbol{\Omega}^2 \mathbf{Z}
= \operatorname{diag}\bigl(\tilde{n}_1,\dots,\tilde{n}_K\bigr),
\]
where $\tilde{n}_k = \sum_{i: z_i = k} \omega_i^2$. Hence the columns of $\mathbf{B}$ are orthogonal and its singular values are $\sqrt{\tilde{n}_1},\dots,\sqrt{\tilde{n}_K}$. Consequently, we have
\[
\sigma_1(\mathbf{B}) = \sqrt{\max_k \tilde{n}_k}, \qquad
\sigma_K(\mathbf{B}) = \sqrt{\min_k \tilde{n}_k}.
\]

The matrix $\mathbf{P}= \boldsymbol{\Theta}\mathbf{B}^\top$. By  standard bounds, we have
\[
\sigma_1(\mathbf{P}) \le \sigma_1(\mathbf{B})\,\sigma_1(\boldsymbol{\Theta}),\qquad
\sigma_K(\mathbf{P}) \ge \sigma_K(\mathbf{B})\,\sigma_K(\boldsymbol{\Theta}).
\]

Now use the bounds on $\tilde{n}_k$ in terms of the cluster sizes $n_k$:
\[
\omega_{\min}^2 n_k \le \tilde{n}_k \le \omega_{\max}^2 n_k,
\]
so that
\[
\sqrt{\min_k \tilde{n}_k} \ge \omega_{\min}\sqrt{\min_k n_k},\qquad
\sqrt{\max_k \tilde{n}_k} \le \omega_{\max}\sqrt{\max_k n_k}.
\]

Thus, we have
\[
\kappa_{\mathbf{P}} = \frac{\sigma_1(\mathbf{P})}{\sigma_K(\mathbf{P})}
\le \frac{\sigma_1(\mathbf{B})\,\sigma_1(\boldsymbol{\Theta})}
{\sigma_K(\mathbf{B})\,\sigma_K(\boldsymbol{\Theta})}
\le \frac{\omega_{\max}\sqrt{\max_k n_k}\,\sigma_1(\boldsymbol{\Theta})}
{\omega_{\min}\sqrt{\min_k n_k}\,\sigma_K(\boldsymbol{\Theta})}
= \kappa\,\frac{\omega_{\max}}{\omega_{\min}}\sqrt{\tau}.
\]
\end{proof}
\bibliographystyle{elsarticle-num}
\bibliography{refIhGMM}

\end{document}